\pdfoutput=1
\documentclass[10pt,twocolumn,letterpaper]{article}

\usepackage{cvpr}
\usepackage{times}
\usepackage{epsfig}
\usepackage{graphicx}
\usepackage{amsmath}
\usepackage{amssymb}

\usepackage[ruled, vlined]{algorithm2e}
\usepackage{amsfonts}       
\usepackage{amsthm}
\usepackage{authblk}
\usepackage{arydshln}
\usepackage{bm}
\usepackage{booktabs}       
\usepackage[T1]{fontenc}    
\usepackage[utf8]{inputenc} 
\usepackage{microtype}      
\usepackage{multirow}
\usepackage{nicefrac}       
\usepackage[caption=false,font=normalsize]{subfig}
\usepackage{threeparttable}
\usepackage{url}            
\usepackage{xcolor}
\usepackage{xspace}
\usepackage{mathtools}

\usepackage[breaklinks=true,bookmarks=false]{hyperref}

\cvprfinalcopy 

\def\cvprPaperID{****} 

\theoremstyle{definition}

\newtheorem{theorem}{Theorem}[section]

\newcommand\figref[1]{Fig.~\ref{#1}}

\newcommand\tabref[1]{Table~\ref{#1}}

\newcommand\secref[1]{Sec.~\ref{#1}}
\newcommand\equref[1]{Eq.(\ref{#1})}

\newcounter{daggerfootnote}

\newcommand{\fakeparagraph}[1]{\vspace{1mm}\noindent\textbf{#1.}}

\newcommand{\sysname}{ALQ\xspace}

\begin{document}

\title{Adaptive Loss-aware Quantization for Multi-bit Networks}

\author[1]{Zhongnan~Qu}
\author[2]{Zimu~Zhou}
\author[1]{Yun~Cheng}
\author[1]{Lothar~Thiele}
\affil[1]{Computer Engineering Group, ETH Zurich, Switzerland \protect\\
{\tt\small \{quz, chengyu, thiele\}@ethz.ch}}
\affil[2]{School of Information Systems, Singapore Management University, Singapore \protect\\
{\tt\small zimuzhou@smu.edu.sg}}


\maketitle

\begin{abstract}
We investigate the compression of deep neural networks by quantizing their weights and activations into multiple binary bases, known as multi-bit networks (MBNs), which accelerate the inference and reduce the storage for the deployment on low-resource mobile and embedded platforms.
We propose Adaptive Loss-aware Quantization (\sysname), a new MBN quantization pipeline that is able to achieve an average bitwidth below one-bit without notable loss in inference accuracy.
Unlike previous MBN quantization solutions that train a quantizer by minimizing the error to reconstruct full precision weights, \sysname directly minimizes the quantization-induced error on the loss function involving neither gradient approximation nor full precision maintenance.
\sysname also exploits strategies including adaptive bitwidth, smooth bitwidth reduction, and iterative trained quantization to allow a smaller network size without loss in accuracy.
Experiment results on popular image datasets show that \sysname outperforms state-of-the-art compressed networks in terms of both storage and accuracy.
Code is available at https://github.com/zqu1992/ALQ

\end{abstract}

\section{Introduction}
\label{sec:introduction}
There is a growing interest to deploy deep neural networks on resource-constrained devices to enable new intelligent services such as mobile assistants, augmented reality, and autonomous cars.
However, deep neural networks are notoriously resource-intensive.
Their complexity needs to be trimmed down to fit in mobile and embedded devices.

To take advantage of the various pretrained models for efficient inference on resource-constrained devices, it is common to compress the pretrained models via pruning \cite{bib:ICLR16:Han}, quantization \cite{bib:arXiv14:Gong, bib:CVPR17:Guo, bib:NIPS17:Lin, bib:ICLR18:Xu, bib:ECCV18:Zhang}, distillation \cite{bib:NIPS14:Hinton}, among others.
We focus on quantization, especially quantizing both the full precision weights and activations of a deep neural network into binary encodes and the corresponding scaling factors \cite{bib:arXiv16:Courbariaux, bib:ECCV16:Rastegari}, which are also interpreted as binary basis vectors and floating-point coordinates in a geometry viewpoint \cite{bib:CVPR17:Guo}.
Neural networks quantized with binary encodes replace expensive floating-point operations by bitwise operations, which are supported even by microprocessors and often result in small memory footprints \cite{bib:ICLR18:Mishra}.
Since the space spanned by only one-bit binary basis and one coordinate is too sparse to optimize, many researchers suggest a multi-bit network (MBN) \cite{bib:arXiv14:Gong, bib:CVPR17:Guo, bib:AAAI18:Hu, bib:NIPS17:Lin, bib:ICLR18:Xu, bib:ECCV18:Zhang}, which allows to obtain a small size without notable accuracy loss and still leverages bitwise operations.
An MBN is usually obtained via trained quantization.
Recent studies~\cite{bib:ICLR18:Pedersoli} leverage bit-packing and bitwise computations for efficient deploying binary networks on a wide range of general devices, which also provides more flexibility to design multi-bit/binary networks. 

Most MBN quantization schemes~\cite{bib:arXiv14:Gong, bib:CVPR17:Guo, bib:AAAI18:Hu, bib:NIPS17:Lin, bib:ICLR18:Xu, bib:ECCV18:Zhang} predetermine a global bitwidth, and learn a quantizer to transform the full precision parameters into binary bases and coordinates such that the quantized models do not incur a significant accuracy loss.
However, these approaches have the following drawbacks:
\begin{itemize}
    \item 
    A global bitwidth may be sub-optimal.
    Recent studies on fixed-point quantization \cite{bib:ICLR18:Khoram, bib:ICML16:Lin} show that the optimal bitwidth varies across layers.
    \item 
    Previous efforts \cite{bib:NIPS17:Lin, bib:ICLR18:Xu, bib:ECCV18:Zhang} retain inference accuracy by minimizing the weight reconstruction error rather than the loss function.
    Such an indirect optimization objective may lead to a notable loss in accuracy.
    Furthermore, they rely on approximated gradients, \eg straight-through estimators (STE)
    to propagate gradients through quantization functions during training.
    \item
    Many quantization schemes~\cite{bib:ECCV16:Rastegari,bib:ECCV18:Zhang} keep the first and last layer in full precision, because quantizing these layers to low bitwidth tends to dramatically decrease the inference accuracy \cite{bib:ECCV18:Wan,bib:ICLR18:Mishra2}. 
    However, these two full precision layers can be a significant storage overhead compared to other low-bit layers (see \secref{sec:imagenet}).
    Also, floating-point operations in both layers can take up the majority of computation in quantized networks~\cite{bib:ICLR19:Louizos}.
\end{itemize}

We overcome the above drawbacks via a novel \textit{A}daptive \textit{L}oss-aware \textit{Q}uantization scheme (\sysname).
Instead of using a uniform bitwidth, \sysname assigns a different bitwidth to each group of weights.
More importantly, \sysname directly minimizes the loss function w.r.t. the quantized weights, by iteratively learning a quantizer that \textit{(i)} smoothly reduces the number of binary bases and \textit{(ii)} alternatively optimizes the remaining binary bases and the corresponding coordinates.
Although loss-aware quantization has been proposed for binary and ternary networks \cite{bib:ICLR17:Hou, bib:ICLR18:Hou, bib:CVPR18:Zhou}, they are inapplicable to MBNs due to the extended optimization space.
They also need approximated gradients during training.
\sysname is the first loss-aware quantization scheme for MBNs and eliminates the need for approximating gradients and retaining full precision weights.
\sysname is also able to quantize the first and last layers without incurring a notable accuracy loss.
The main contributions of this work are as follows.
\begin{itemize}
  \item
    We design \sysname, the first loss-aware quantization scheme for multi-bit networks. 
    It is also the first trained quantizer without gradient approximation, and realizes an adaptive bitwidth w.r.t the loss for MBNs (including the first and last layers).
  \item
    \sysname enables extremely low-bit (yet dense tensor form) binary networks with an average bitwidth below 1-bit. 
    Experiments on CIFAR10 show that \sysname can compress VGG to an average bitwidth of $0.4$-bit, while yielding a higher accuracy than other binary networks~\cite{bib:ECCV16:Rastegari,bib:arXiv16:Courbariaux}.
\end{itemize}

\section{Related Work}
\label{sec:related}
\sysname follows the trend to quantize deep neural networks using discrete bases to reduce expensive floating-point operations.
Commonly used bases include fixed-point~\cite{bib:arXiv16:Zhou}, power of two \cite{bib:JMLR17:Hubara, bib:ICLR17:Zhou}, and $\{-1,0,+1\}$ \cite{bib:arXiv16:Courbariaux, bib:ECCV16:Rastegari}.
We focus on quantization with binary bases \ie $\{-1,+1\}$ among others for the following considerations. 
\textit{(i)} 
If both weights and activations are quantized with the same binary basis, it is possible to evaluate 32 multiply-accumulate operations (MACs) with only 3 instructions on a 32-bit microprocessor, \ie bitwise $\texttt{xnor}$, $\texttt{popcount}$, and accumulation.
This will significantly speed up the convolution operations \cite{bib:JMLR17:Hubara}.
\textit{(ii)}
A network quantized to fixed-point requires specialized integer arithmetic units (with various bitwidth) for efficient computing~\cite{bib:MICRO17:Albericio,bib:ICLR18:Khoram}, whereas a network quantized with multiple binary bases adopts the same operations mentioned before as binary networks.
Popular networks quantized with binary bases include \textit{Binary Networks} and \textit{Multi-bit Networks}.

\subsection{Quantization for Binary Networks}
BNN \cite{bib:arXiv16:Courbariaux} is the first network with both binarized weights and activations.
It dramatically reduces the memory and computation but often with notable accuracy loss.
To resume the accuracy degradation from binarization, XNOR-Net \cite{bib:ECCV16:Rastegari} introduces a layer-wise full precision scaling factor into BNN.
However, XNOR-Net leaves the first and last layers unquantized, which consumes more memory.
SYQ \cite{bib:CVPR18:Faraone} studies the efficiency of different structures during binarization/ternarization. 
LAB \cite{bib:ICLR17:Hou} is the first loss-aware quantization scheme which optimizes the weights by directly minimizing the loss function. 

\sysname is inspired by recent loss-aware binary networks such as LAB \cite{bib:ICLR17:Hou}.
Loss-aware quantization has also been extended to fixed-point networks in \cite{bib:ICLR18:Hou}.
However, existing loss-aware quantization schemes \cite{bib:ICLR17:Hou,bib:ICLR18:Hou} are inapplicable for MBNs.
This is because multiple binary bases dramatically extend the optimization space with the same bitwidth (\ie an optimal set of binary bases rather than a single basis), which may be intractable. 
Some proposals \cite{bib:ICLR17:Hou, bib:ICLR18:Hou, bib:CVPR18:Zhou} still require full-precision weights and gradient approximation (backward STE and forward loss-aware projection), introducing undesirable errors when minimizing the loss.
In contrast, \sysname is free from gradient approximation.

\subsection{Quantization for Multi-bit Networks}
MBNs denote networks that use multiple binary bases to trade off storage and accuracy.
Gong \etal propose a residual quantization process, which greedily searches the next binary basis by minimizing the residual reconstruction error~\cite{bib:arXiv14:Gong}.
Guo \etal improve the greedy search with a least square refinement~\cite{bib:CVPR17:Guo}.
Xu \etal~\cite{bib:ICLR18:Xu} separate this search into two alternating steps, fixing coordinates then exhausted searching for optimal bases, and fixing the bases then refining the coordinates using the method in \cite{bib:CVPR17:Guo}.
LQ-Net~\cite{bib:ECCV18:Zhang} extends the scheme of~\cite{bib:ICLR18:Xu} with a moving average updating, which jointly quantizes weights and activations.
However, similar to XNOR-Net \cite{bib:ECCV16:Rastegari}, LQ-Net~\cite{bib:ECCV18:Zhang} does not quantize the first and last layers.
ABC-Net~\cite{bib:NIPS17:Lin} leverages the statistical information of all weights to construct the binary bases as a whole for all layers. 

All the state-of-the-art MBN quantization schemes minimize the weight reconstruction error rather than the loss function of the network. 
They also rely on the gradient approximation such as STE when back propagating the quantization function.
In addition, they all predetermine a uniform bitwidth for all parameters. 
The indirect objective, the approximated gradient, and the global bitwidth lead to a sub-optimal quantization.
\sysname is the first scheme to explicitly optimize the loss function and incrementally train an adaptive bitwidth while without gradient approximation.

\section{Adaptive Loss-Aware Quantization}
\label{sec:weight}
\subsection{Weight Quantization Overview}
\label{subsec:overview}
\fakeparagraph{Notations}
We aim at MBN quantization with an adaptive bitwidth.
To allow adaptive bitwidth, we structure the weights in \textit{disjoint groups}.
Specifically, for the vectorized weights $\bm{w}$ of a given layer $l$, where $\bm{w}\in\mathbb{R}^{N\times1}$, we divide $\bm{w}$ into $G$ disjoint groups.
For simplicity, we omit the subscript $l$. 
Each group of weights is denoted by $\bm{w}_{g}$, where $\bm{w}_{g}\in\mathbb{R}^{n\times1}$ and $N = n \times G$.
Then the quantized weights of each group, $\bm{\hat{w}}_g = \sum_{i=1}^{I_g}\alpha_i\bm{\beta}_i=\bm{B}_g\bm{\alpha}_g$. 
$\bm{\beta}_i\in\{-1,+1\}^{n\times 1}$ and $\alpha_i\in\mathbb{R}_+$ are the $i^{\text{th}}$ binary basis and the corresponding coordinate; $I_g$ represents the bitwidth, \ie the number of binary bases, of group $g$.
$\bm{B}_g\in\{-1,+1\}^{n\times I_g}$ and $\bm{\alpha}_g\in\mathbb{R}_+^{I_g\times1}$ are the matrix forms of the binary bases and the coordinates.
We further denote $\bm{\alpha}$ as vectorized coordinates $\{\bm{\alpha}_g\}_{g=1}^G$, and $\bm{B}$ as concatenated binary bases $\{\bm{B}_g\}_{g=1}^G$ of all weight groups in layer $l$.
A layer $l$ quantized as above yields an average bitwidth $I = \frac{1}{G}\sum_{g = 1}^G I_g$.
We discuss the choice of group size $n$, and the initial $\bm{B}_g$, $\bm{\alpha}_g$, $I_g$ in \secref{sec:structure}.

\fakeparagraph{Problem Formulation}
\sysname quantizes weights by directly minimizing the loss function rather than the reconstruction error.
For layer $l$, the process can be formulated as the following optimization problem.
\begin{eqnarray}
  \min_{\bm{\hat{w}}_g} & & \ell\left(\bm{\hat{w}}_g\right) \label{eq:objective} \\
  \text{s.t.} & & \bm{\hat{w}}_g = \sum_{i=1}^{I_g}\alpha_i\bm{\beta}_i = \bm{B}_g\bm{\alpha}_g
  \label{eq:weights}
  \\& & \mathrm{card}(\bm{\alpha}) = I\times G \leq \mathrm{I_{min}}\times G
  \label{eq:sumIg}
\end{eqnarray}
where $\ell$ is the loss; $\mathrm{card}(.)$ denotes the cardinality of the set, \ie the total number of elements in $\bm{\alpha}$; $\mathrm{I_{min}}$ is the desirable average bitwidth.
Since the group size $n$ is the same in one layer, $\mathrm{card}(\bm{\alpha})$ is proportional to the storage consumption.

\sysname tries to solve the optimization problem in \equref{eq:objective}-\equref{eq:sumIg} by \textit{iteratively} solving two sub-problems as below.
The overall pseudocode is illustrated in Alg.~\ref{alg:pipeline} in Appendix~\ref{sec:pipeline}.

\begin{itemize}
  \item 
  \textbf{Step 1: Pruning in $\bm{\alpha}$ Domain} (\secref{sec:pruning}).
  In this step, we progressively reduce the average bitwidth $I$ for a layer $l$ by pruning the least important (w.r.t. the loss) coordinates in $\bm{\alpha}$ domain.
  Note that removing an element $\alpha_i$ will also lead to the removal of the binary basis $\bm{\beta}_i$, which in effect results in a smaller bitwidth $I_g$ for group $g$.
  This way, no sparse tensor is introduced. 
  Sparse tensors could lead to a detrimental irregular computation.
  Since the importance of each weight group differs, the resulting $I_g$ varies across groups, and thus contributes to an adaptive bitwidth $I_g$ for each group.
  In this step, we only set some elements of $\bm{\alpha}$ to zero (also remove them from $\bm{\alpha}$ leading to a reduced $I_g$) without changing the others.
  The optimization problem for Step 1 is:
  \begin{eqnarray}
  \min_{\bm{\alpha}} & & \ell\left(\bm{\alpha}\right) \label{eq:pruning:obj} \\
  \text{s.t.} & & \mathrm{card}(\bm{\alpha}) \leq \mathrm{I_{min}}\times G \label{eq:pruning:constraint}
  \end{eqnarray}
  \item

  \textbf{Step 2: Optimizing Binary Bases $\bm{B}_g$ and Coordinates $\bm{\alpha}_g$} (\secref{sec:updating}).
  In this step, we retrain the remaining binary bases and coordinates to recover the accuracy degradation induced by the bitwidth reduction.
  Similar to~\cite{bib:ICLR18:Xu}, we take an alternative approach for better accuracy recovery.
  Specifically, we first search for a new set of binary bases w.r.t. the loss given fixed coordinates.
  Then we optimize the coordinates by fixing the binary bases.
  The optimization problem for Step 2 is:
  \begin{eqnarray}
  \min_{\bm{\hat{w}}_g} & & \ell\left(\bm{\hat{w}}_g\right) \label{eq:updating:obj}\\
  \text{s.t.} & & \bm{\hat{w}}_g = \sum_{i=1}^{I_g}\alpha_i\bm{\beta}_i=\bm{B}_g\bm{\alpha}_g \label{eq:updating:constraint}
  \end{eqnarray}
\end{itemize}

\fakeparagraph{Optimizer Framework}
We consider both sub-problems above as an optimization problem with \textit{domain constraints}, and solve them using the same optimization framework: subgradient methods with projection update \cite{bib:JMLR11:Duchi}.

The optimization problem in \equref{eq:updating:obj}-\equref{eq:updating:constraint} imposes domain constraints on $\bm{B}_g$ because they can only be discrete binary bases.
The optimization problem in \equref{eq:pruning:obj}-\equref{eq:pruning:constraint} can be considered as with a trivial domain constraint: the output $\bm{\alpha}$ should be a subset (subvector) of the input $\bm{\alpha}$.
Furthermore, the feasible sets for both $\bm{B}_g$ and $\bm{\alpha}$ are bounded.

Subgradient methods with projection update are effective to solve problems in the form of $\min_{\bm{x}}(\ell(\bm{x}))$ s.t. $\bm{x}\in\mathbb{X}$ \cite{bib:JMLR11:Duchi}.
We apply AMSGrad~\cite{bib:ICLR18:Reddi}, an adaptive stochastic subgradient method with projection update, as the common optimizer framework in the two steps.
At iteration $s$, AMSGrad generates the next update as,
\begin{equation}
\begin{split}
\bm{x}^{s+1} & = \Pi_{\mathbb{X},\sqrt{\bm{\hat{V}}^s}}(\bm{x}^s-a^s\bm{m}^s/\sqrt{\bm{\hat{v}}^s}) \\
             & = \underset{{\bm{x}\in\mathbb{X}}}{\mathrm{argmin}}~\|(\sqrt{\bm{\hat{V}}^s})^{1/2}(\bm{x}-(\bm{x}^s-\frac{a^s\bm{m}^s}{\sqrt{\bm{\hat{v}}^s}}))\|
\end{split}
\label{eq:amsgradx}
\end{equation}
where $\Pi$ is a projection operator; $\mathbb{X}$ is the feasible domain of $\bm{x}$; $a^s$ is the learning rate; $\bm{m}^s$ is the (unbiased) first momentum; $\bm{\hat{v}}^s$ is the (unbiased) maximum second momentum; and $\bm{\hat{V}}^s$ is the diagonal matrix of $\bm{\hat{v}}^s$.

In our context, \equref{eq:amsgradx} can be written as,
\begin{equation}
\bm{\hat{w}}_g^{s+1} = \underset{\bm{\hat{w}}_g\in\mathbb{F}}{\mathrm{argmin}} f^s(\bm{\hat{w}}_g)
\label{eq:amsgradw2}
\end{equation}
\begin{equation}
f^s=(a^s\bm{m}^s)^{\mathrm{T}}(\bm{\hat{w}}_g-\bm{\hat{w}}_g^s)+\frac{1}{2}(\bm{\hat{w}}_g-\bm{\hat{w}}_g^s)^{\mathrm{T}}\sqrt{\bm{\hat{V}}^s}(\bm{\hat{w}}_g-\bm{\hat{w}}_g^s)
\label{eq:amsgradw3}
\end{equation}
where $\mathbb{F}$ is the feasible domain of $\bm{\hat{w}}_g$.

Step 1 and Step 2 have different feasible domains of $\mathbb{F}$ according to their objective (details in~\secref{sec:pruning} and~\secref{sec:updating}).
\equref{eq:amsgradw3} approximates the loss increment incurred by $\bm{\hat{w}}_g$ around the current point $\bm{\hat{w}}_g^s$ as a quadratic model function under domain constraints \cite{bib:JMLR11:Duchi,bib:ICLR18:Reddi}.
For simplicity, we replace $a^s\bm{m}^s$ with $\bm{g}^s$ and replace $\sqrt{\bm{\hat{V}}^s}$ with $\bm{H}^s$.
$\bm{g}^s$ and $\bm{H}^s$ are iteratively updated by the loss gradient of $\bm{\hat{w}}_g^s$. 
Thus, the required input of each AMSGrad step is $\frac{\partial\ell^s}{\partial {\bm{\hat{w}}_g}^s}$.
Since $\bm{\hat{w}}_g^s$ is used as an intermediate value during the forward, it can be directly obtained during the backward.

\subsection{Pruning in $\bm{\alpha}$ Domain}
\label{sec:pruning}

As introduced in \secref{subsec:overview}, we reduce the bitwidth $I$ by pruning the elements in $\bm{\alpha}$ w.r.t. the resulting loss.
If one element $\alpha_i$ in $\bm{\alpha}$ is pruned, the corresponding dimension $\bm{\beta}_i$ is also removed from $\bm{B}$.
Now we explain how to instantiate the optimizer in \equref{eq:amsgradw2} to solve \equref{eq:pruning:obj}-\equref{eq:pruning:constraint} of Step 1.

The cardinality of the chosen subset (\ie the average bitwidth) is uniformly reduced over iterations. 
For example, assume there are $T$ iterations in total, the initial average bitwidth is $I^0$ and the desired average bitwidth after $T$ iterations $I^{T}$ is $\mathrm{I_{min}}$.
Then at each iteration $t$, ($M_p = \mathrm{round}((I^{0}-\mathrm{I_{min}})\times G/T)$) of $\alpha_i^t$'s are pruned in this layer. 
This way, the cardinality after $T$ iterations will be smaller than $\mathrm{I_{min}}\times G$.
See Alg.~\ref{alg:pruning} in Appendix~\ref{sec:pseudocodepruning} for the pseudocode.

When pruning in the $\bm{\alpha}$ domain, $\bm{B}$ is considered as invariant. 
Hence \equref{eq:amsgradw2} and \equref{eq:amsgradw3} become,
\begin{equation}
\bm{\alpha}^{t+1} = \underset{\bm{\alpha}\in\mathbb{P}}{\mathrm{argmin}}~f_{\bm{\alpha}}^t(\bm{\alpha})
\label{eq:amsgradalpha1}
\end{equation}
\begin{equation}
f_{\bm{\alpha}}^t=(\bm{g}_{\bm{\alpha}}^t)^{\mathrm{T}}
(\bm{\alpha}-\bm{\alpha}^t)
+\frac{1}{2}
(\bm{\alpha}-\bm{\alpha}^t)^{\mathrm{T}}
\bm{H_\alpha}^t
(\bm{\alpha}-\bm{\alpha}^t)
\label{eq:amsgradalpha2}
\end{equation}
where $\bm{g}_{\bm{\alpha}}^t$ and $\bm{H_\alpha}^t$ are similar as in \equref{eq:amsgradw3} but are in the $\bm{\alpha}$ domain.
If $\alpha_i^t$ is pruned, the $i^{\text{th}}$ element in $\bm{\alpha}$ is set to $0$ in the above~\equref{eq:amsgradalpha1} and~\equref{eq:amsgradalpha2}. 
Thus, the constrained domain $\mathbb{P}$ is taken as all possible vectors with $M_p$ zero elements in $\bm{\alpha}^t$. 

AMSGrad uses a diagonal matrix of $\bm{H_\alpha}^t$ in the quadratic model function, which decouples each element in $\bm{\alpha}^t$.
This means the loss increment caused by several $\alpha_i^t$ equals the sum of the increments caused by them individually, which are calculated as,
\begin{equation}
f_{\bm{\alpha},i}^t = -g_{\bm{\alpha},i}^t~\alpha_i^t+\frac{1}{2}~H_{\bm{\alpha},{ii}}^t~({\alpha_i^t})^2
\label{eq:taylorpruning}
\end{equation}
All items of $f_{\bm{\alpha},i}^t$ are sorted in ascending.
Then the first $M_p$ items ($\alpha_i^t$) in the sorted list are removed from $\bm{\alpha}^t$, and results in a smaller cardinality $I^{t}\times G$. 
The input of the AMSGrad step in $\bm{\alpha}$ domain is the loss gradient of $\bm{\alpha}_g^t$, which can be computed with the chain rule,
\begin{equation}
\frac{\partial\ell^t}{\partial\bm{\alpha}_g^t}={\bm{B}_g^t}^{\mathrm{T}} \frac{\partial\ell^t}{\partial {\bm{\hat{w}}_g}^t}
\label{eq:gradientspruning}
\end{equation}
\begin{equation}
\bm{\hat{w}}_g^t=\bm{B}_g^t \bm{\alpha}_g^t
\end{equation}

Our pipeline allows to reduce the bitwidth smoothly, since the average bitwidth can be floating-point.
In \sysname, since different layers have a similar group size (see \secref{sec:structure}), the loss increment caused by pruning is sorted among all layers, such that only a global pruning number needs to be determined.
The global pruning number is defined by the total number of pruned $\alpha_i$'s, \ie the difference of $\sum_l{\mathrm{card}(\bm{\alpha}_l)}$ before and after pruning.
More details are explained in Appendix~\ref{sec:pseudocodepruning} and \ref{sec:pipeline}.
This pruning step not only provides a loss-aware adaptive bitwidth, but also seeks a better initialization for training the following lower bitwidth quantization, since quantized weights may be relatively far from their original full precision values.

\subsection{Optimizing Binary Bases and Coordinates}
\label{sec:updating}

After pruning, the loss degradation needs to be recovered. 
Following~\equref{eq:amsgradw2}, the objective in Step 2 is
\begin{equation}
\bm{\hat{w}}_g^{s+1} = \underset{\bm{\hat{w}}_g\in\mathbb{F}}{\mathrm{argmin}}~f^s(\bm{\hat{w}}_g)
\end{equation}
The constrained domain $\mathbb{F}$ is decided by both binary bases and full precision coordinates.
Hence directly searching optimal $\bm{\hat{w}}_g$ is NP-hard.
Instead, we optimize $\bm{B}_g$ and $\bm{\alpha}_g$ in an alternative manner, as with prior MBN quantization w.r.t. the reconstruction error \cite{bib:ICLR18:Xu,bib:ECCV18:Zhang}.

\fakeparagraph{Optimizing $\bm{B}_g$}
We directly search the optimal bases with AMSGrad.
In each optimizing iteration $q$, we fix $\bm{\alpha}_g^q$, and update $\bm{B}_g^q$.
We find the optimal increment for each group of weights, such that it converts to a new set of binary bases, $\bm{B}_g^{q+1}$.
This optimization step searches a new space spanned by $\bm{B}_g^{q+1}$ based on the loss reduction, which prevents the pruned space to be always a subspace of the previous one.
See Alg.~\ref{alg:optimizingbases} in Appendix~\ref{sec:pseudocodeB} for the detailed pseudocode.

According to~\equref{eq:amsgradw2} and~\equref{eq:amsgradw3}, the optimal $\bm{B}_g$ w.r.t. the loss is updated by,
\begin{equation}
\bm{B}_g^{q+1} = \underset{\bm{B}_g\in\{-1,+1\}^{n\times I_g}}{\mathrm{argmin}}~f^q(\bm{B}_g)
\label{eq:amsgradB}
\end{equation}
\begin{equation}
\begin{split}
f^q=&~(\bm{g}^q)^{\mathrm{T}}(\bm{B}_g\bm{\alpha}_g^{q}-\bm{\hat{w}}_g^q)+ \\
    &~\frac{1}{2} (\bm{B}_g\bm{\alpha}_g^{q}-\bm{\hat{w}}_g^q)^{\mathrm{T}} \bm{H}^q (\bm{B}_g\bm{\alpha}_g^{q}-\bm{\hat{w}}_g^q)
\label{eq:amsgradB2}
\end{split}
\end{equation}
where $\bm{\hat{w}}_g^q = \bm{B}_g^{q}\bm{\alpha}_g^{q}$.

Since $\bm{H}^q$ is diagonal in AMSGrad, each row vector in $\bm{B}_g^{q+1}$ can be independently determined.
For example, the $j^{\text{th}}$ row is computed as,
\begin{equation}
\bm{B}_{g,j}^{q+1} = \underset{\bm{B}_{g,j}}{\mathrm{argmin}}~\|\bm{B}_{g,j}\bm{\alpha}_{g}^q-(\hat{w}_{g,j}^q-g^q_j/H_{jj}^q)\|
\label{eq:rowupdating}
\end{equation}
In general, $n>>I_g$.
For each group, we firstly compute all $2^{I_g}$ possible values of
\begin{equation}
\bm{b}^{\mathrm{T}}\bm{\alpha}_{g}^q~,~~~ \bm{b}^{\mathrm{T}}\in\{-1,+1\}^{1\times I_g}
\label{eq:comb}
\end{equation}
Then each row vector $\bm{B}_{g,j}^{q+1}$ can be directly assigned by the optimal $\bm{b}^{\mathrm{T}}$ through exhaustive search.

\fakeparagraph{Optimizing $\bm{\alpha}_g$}
The above obtained set of binary bases $\bm{B}_g$ spans a new linear space.
The current $\bm{\alpha}_g$ is unlikely to be a (local) optimal w.r.t. the loss in this space, so now we optimize $\bm{\alpha}_g$.
Since $\bm{\alpha}_g$ is full precision, \ie $\bm{\alpha}_g\in\mathbb{R}^{I_g\times1}$, there is no domain constraint and thus no need for projection updating.
Optimizing full precision $\bm{w}_g$ takes incremental steps in original $n$-dim full space (spanned by orthonormal bases). 
Similarly, optimizing $\bm{\alpha}_g$ searches steps in a $I_g$-dim subspace (spanned by $\bm{B}_g$). 
Hence conventional training strategies can be directly used to optimize $\bm{\alpha}_g$.
See Alg.~\ref{alg:optimizingcoordinates} in Appendix~\ref{sec:pseudocodealpha} for the detailed pseudocode.

Similar as~\equref{eq:amsgradalpha1} and~\equref{eq:amsgradalpha2}, we construct an AMSGrad optimizer in $\bm{\alpha}$ domain but without projection updating, for each group in the $p^{\text{th}}$ iteration as,
\begin{equation}
\bm{\alpha}_g^{p+1} = \bm{\alpha}_g^p-a_{\bm{\alpha}}^p\bm{m}_{\bm{\alpha}}^p/\sqrt{\bm{\hat{v}_\alpha}^p}
\label{eq:optimizingalpha}
\end{equation}

We also add an L2-norm regularization on $\bm{\alpha}_g$ to enforce unimportant coordinates to zero. 
If there is a negative value in $\bm{\alpha}_{g}$, the corresponding basis is set to its negative complement, to keep $\bm{\alpha}_{g}$ semi-positive definite. Optimizing $\bm{B}_g$  and $\bm{\alpha}_g$ does not influence the number of binary bases $I_g$.

\fakeparagraph{Optimization Speedup}
Since $\bm{\alpha}_g$ is full precision, updating $\bm{\alpha}_g^q$ is much cheaper than exhaustively search $\bm{B}_g^{q+1}$. 
Even if the main purpose of the first step in~\secref{sec:updating} is optimizing bases, we also add an updating process for $\bm{\alpha}_g^q$ in each optimizing iteration $q$.

We fix $\bm{B}_{g}^{q+1}$, and update $\bm{\alpha}_{g}^{q}$.
The overall increment of quantized weights from both updating processes is,
\begin{equation}
\bm{\hat{w}}^{q+1}_g - \bm{\hat{w}}^q_{g} = \bm{B}_{g}^{q+1}\bm{\alpha}_{g}^{q+1}-\bm{B}_{g}^{q}\bm{\alpha}_{g}^{q}
\label{eq:incrw}
\end{equation}
Substituting~\equref{eq:incrw} into~\equref{eq:amsgradw2} and~\equref{eq:amsgradw3}, we have,
\begin{equation}
\begin{split}
\bm{\alpha}_{g}^{q+1} = -&((\bm{B}_{g}^{q+1})^{\mathrm{T}} \bm{H}^q \bm{B}_{g}^{q+1})^{-1}\times \\
                        &~((\bm{B}_{g}^{q+1})^{\mathrm{T}}(\bm{g}^q-\bm{H}^q\bm{B}^q_{g}\bm{\alpha}_{g}^{q}))
\label{eq:alphaIncr}
\end{split}
\end{equation}
To ensure the inverse in~\equref{eq:alphaIncr} exists, we add $\lambda \mathbf{I}$ onto $(\bm{B}_{g}^{q+1})^{\mathrm{T}} \bm{H}^q \bm{B}_{g}^{q+1}$, where $\lambda=10^{-6}$.

\section{Activation Quantization}
\label{sec:activationQuantization}

To leverage bitwise operations for speedup, the inputs of each layer (\ie the activation output of the last layer) also need to be quantized into the multi-bit form.
Unlike previous works \cite{bib:ECCV18:Zhang} that quantize activations with a different binary basis ($\{0,+1\}$) as weights, we also quantize activations with $\{-1,+1\}$.
This way, we only need 3 instructions rather than 5 instructions to replace the original 32 MACs (see \secref{sec:related}).

Our activation quantization follows the idea proposed in \cite{bib:arXiv18:Choi}, \ie a parameterized clipping for fixed-point activation quantization, but it is adapted to the multi-bit form.
Specially, we replace ReLu with a step activation function.
The vectorized activation $\bm{x}$ of the $l^{\text{th}}$ layer is quantized as,
\begin{equation}
\bm{x}\doteq\bm{\hat{x}}=x_{ref}+\bm{D}\bm{\gamma}=\bm{D}'\bm{\gamma}'
\label{eq:act}
\end{equation}
where $\bm{D}\in\{-1,+1\}^{N_x\times I_x}$, and $\bm{\gamma}\in\mathbb{R}_+^{I_x\times1}$.
$\bm{\gamma}'$ is a column vector formed by $[x_{ref},\bm{\gamma}^\mathrm{T}]^{\mathrm{T}}$; $\bm{D}'$ is a matrix formed by $[\bm{1}^{{N_x\times 1}}, \bm{D}]$.
$N_x$ is the dimension of $\bm{x}$, and $I_x$ is the quantization bitwidth for activations.
$x_{ref}$ is the introduced layerwise (positive floating-point) reference to fit in the output range of ReLu.
During inference, $x_{ref}$ is convoluted with the weights of the next layer and added to the bias.
Hence the introduction of $x_{ref}$ does not lead to extra computations. 
The output of the last layer is not quantized, as it does not involve computations anymore.
For other settings, we directly adopt them used in \cite{bib:ECCV18:Zhang}. 
$\bm{\gamma}$ and $x_{ref}$ are updated during the forward propagation with a running average to minimize the squared reconstruction error as,
\begin{equation}
\bm{\gamma}'_{new} = (\bm{D'}^{\mathrm{T}}\bm{D}')^{-1}\bm{D'}^{\mathrm{T}}\bm{x}
\end{equation}
\begin{equation}
\bm{\gamma}' = 0.9\bm{\gamma}'+(1-0.9)\bm{\gamma}'_{new}
\end{equation}

The (quantized) weights are also further fine-tuned with our optimizer to resume the accuracy drop.
Here, we only set a global bitwidth for all layers in activation quantization.

\section{Experiments}
\label{sec:experiment}
We implement \sysname with Pytorch~\cite{bib:NIPSWorkshop17:Paszke}, and evaluate its performance on MNIST~\cite{bib:MNIST}, CIFAR10~\cite{cifar10}, and ILSVRC12 (ImageNet)~\cite{ILSVRC15} using LeNet5~\cite{bib:PIEEE98:LeCun}, VGG~\cite{bib:ICLR17:Hou,bib:ECCV16:Rastegari}, and ResNet18/34~\cite{bib:CVPR16:He}, respectively.
More implementation details are provided in Appendix~\ref{sec:implmentationDetails}.

\subsection{\sysname Initialization}
\label{sec:structure}
We adapt the network sketching proposed in~\cite{bib:CVPR17:Guo} for $\bm{\hat{w}}_g$ initialization, and realize a structured sketching (see Alg.~\ref{alg:SS} in Appendix~\ref{sec:initalgorithm}).
Some important parameters in Alg.~\ref{alg:SS} are chosen as below. 

\fakeparagraph{Group Size $n$} 
We empirically decide a range for the group size $n$ by trading off between the weight reconstruction error and the storage compression rate.
A group size from $32$ to $512$ achieves a good balance.
Accordingly, for a convolution layer, grouping in channel-wise ($\bm{w}_{c,:,:,:}$), kernel-wise ($\bm{w}_{c,d,:,:}$), and pixel-wise ($\bm{w}_{c,:,h,w}$) appears to be appropriate.
Channel-wise $\bm{w}_{c,:}$ and subchannel-wise $\bm{w}_{c,d:d+n}$ grouping are suited for a fully connected layer.
In addition, the most frequently used structures for current popular networks are pixel-wise (convolution layers) and (sub)channel-wise (fully connected layers), which align with the bit-packing approach in~\cite{bib:ICLR18:Pedersoli}.
See Appendix~\ref{sec:group} for more details on grouping.

\fakeparagraph{Maximum Bitwidth $\mathrm{I_{max}}$ for Group $g$}
The initial $I_g$ is set by a predefined initial reconstruction precision or a maximum bitwidth.
We notice that the accuracy degradation caused by the initialization can be fully recovered after several optimization epochs proposed in \secref{sec:updating}, if the maximum bitwidth is $8$.
For example, ResNet18 on ILSVRC12 after such an initialization can be retrained to a Top-1/5 accuracy of $70.3\%$/$89.4\%$, even higher than its full precision counterpart ($69.8\%$/$89.1\%$). 
For smaller networks, \eg VGG on CIFAR10, a maximum bitwidth of $6$ is sufficient.

\subsection{Convergence Analysis}
\label{sec:convergence}
\fakeparagraph{Settings}
This experiment conducts the ablation study of our optimization step in \secref{sec:updating}.
We show the advantages of our optimizer in terms of convergence, on networks quantized with a uniform bitwidth.
Optimizing $\bm{B}_g$ with speedup (also Alg.~\ref{alg:optimizingbases}) is compared, since it takes a similar alternating step as previous works~\cite{bib:ICLR18:Xu,bib:ECCV18:Zhang}. 
Recall that our optimizer \textit{(i)} has no gradient approximation and \textit{(ii)} directly minimizes the loss.
We use AMSGrad\footnote{AMSGrad can also optimize full precision parameters.} with a learning rate of 0.001, and compare with following baselines.

\begin{itemize}
    \item 
    \textit{STE with rec. error:}
    This baseline quantizes the maintained full precision weights by minimizing the reconstruction error (rather than the loss) during forward and approximates gradients via STE during backward.
    This approach is adopted in some of the best-performing quantization schemes such as LQ-Net \cite{bib:ECCV18:Zhang}.
    \item 
    \textit{STE with loss-aware:}
    This baseline approximates gradients via STE but performs a loss-aware projection updating (adapted from our \sysname).
    It can be considered as a multi-bit extension of prior loss-aware quantizers for binary and ternary networks \cite{bib:ICLR17:Hou,bib:ICLR18:Hou}.
    See Alg.~\ref{alg:ste} in Appendix~\ref{sec:ste} for the detailed pseudocode.
\end{itemize}

\begin{figure}[htbp!]
\centering
\includegraphics[width=0.45\textwidth]{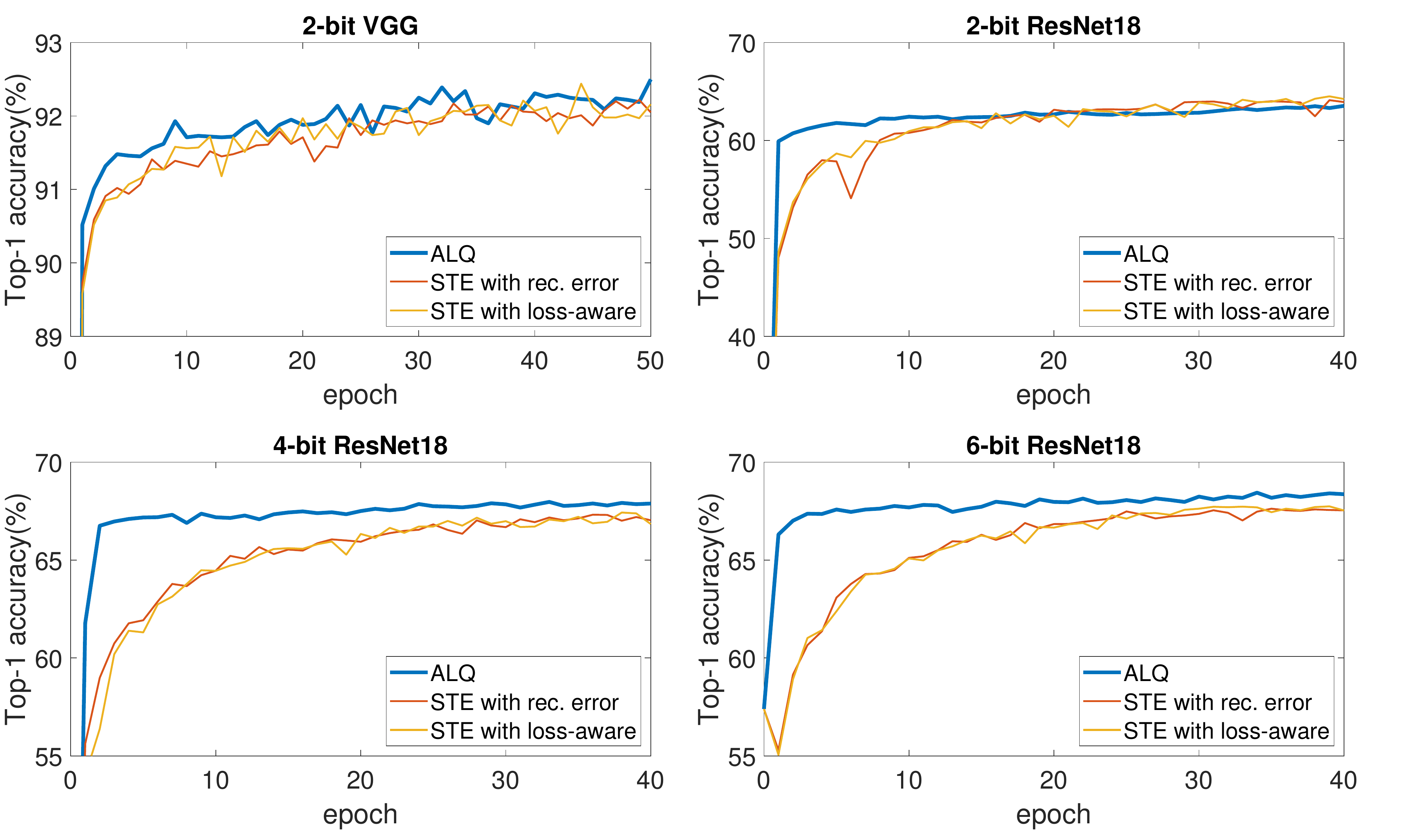}
\caption{Validation accuracy trained with \sysname/baselines.}
\label{fig:convergence}
\end{figure}

\fakeparagraph{Results}
\figref{fig:convergence} shows the Top-1 validation accuracy of different optimizers, with increasing epochs on uniform bitwidth MBNs.
\sysname exhibits not only a more stable and faster convergence, but also a higher accuracy.
The exception is 2-bit ResNet18. 
\sysname converges faster, but the validation accuracy trained with STE gradually exceeds \sysname after about 20 epochs.
For training a large network with $\leq2$ bitwidth, the positive effect brought from the high precision trace may compensate certain negative effects caused by gradient approximation.
In this case, keeping full precision parameters will help calibrate some aggressive steps of quantization, resulting in a slow oscillating convergence to a better local optimum.
This also encourages us to add several epochs of STE based optimization (\eg \textit{STE with loss-aware}) after low bitwidth quantization to further regain the accuracy.

\subsection{Effectiveness of Adaptive Bitwidth}
\fakeparagraph{Settings}
This experiment demonstrates the performance of incrementally trained adaptive bitwidth in \sysname, \ie our pruning step in \secref{sec:pruning}.
Uniform bitwidth quantization (an equal bitwidth allocation across all groups in all layers) is taken as the baseline.
The baseline is trained with the same number of epochs as the sum of all epochs during the bitwidth reduction.
Both \sysname and the baseline are trained with the same learning rate decay schedule.

\fakeparagraph{Results}
\tabref{tab:adapt} shows that there is a large Top-1 accuracy gap between an adaptive bitwidth trained with \sysname and a uniform bitwidth (baseline).
In addition to the overall average bitwidth ($I_W$), we also plot the distribution of the average bitwidth and the number of weights across layers (both models in \tabref{tab:adapt}) in \figref{fig:adapt}.
Generally, the first several layers and the last layer are more sensitive to the loss, thus require a higher bitwidth. 
The shortcut layers in ResNet architecture (\eg the $8^{\text{th}}$, $13^{\text{rd}}$, $18^{\text{th}}$ layers in ResNet18) also need a higher bitwidth.
We think this is due to the fact that the shortcut pass helps the information forward/backward propagate through the blocks. 
Since the average of adaptive bitwidth can have a decimal part, \sysname can achieve a compression rate with a much higher resolution than a uniform bitwidth, which not only controls a more precise trade-off between storage and accuracy, but also benefits our incremental bitwidth reduction (pruning) scheme.

\begin{table}[htbp!]
\centering
\caption{Comparison between Baseline (Uniform Bitwidth) and \sysname (Adaptive Bitwidth)}
\label{tab:adapt}
\small
\begin{tabular}{ccc}
\toprule
Method                                      & $I_W$                    & Top-1              \\ \hline
Baseline VGG (uniform)                      & 1                        & 91.8\%             \\
\textbf{\sysname VGG}                       & \textbf{0.66}            & \textbf{92.0}\%    \\
Baseline ResNet18 (uniform)                 & 2                        & 66.2\%             \\
\textbf{\sysname ResNet18}                  & \textbf{2.00}             & \textbf{68.9}\%    \\
\bottomrule
\end{tabular}
\end{table}

\begin{figure}[htbp!]
\centering
\includegraphics[width=0.45\textwidth]{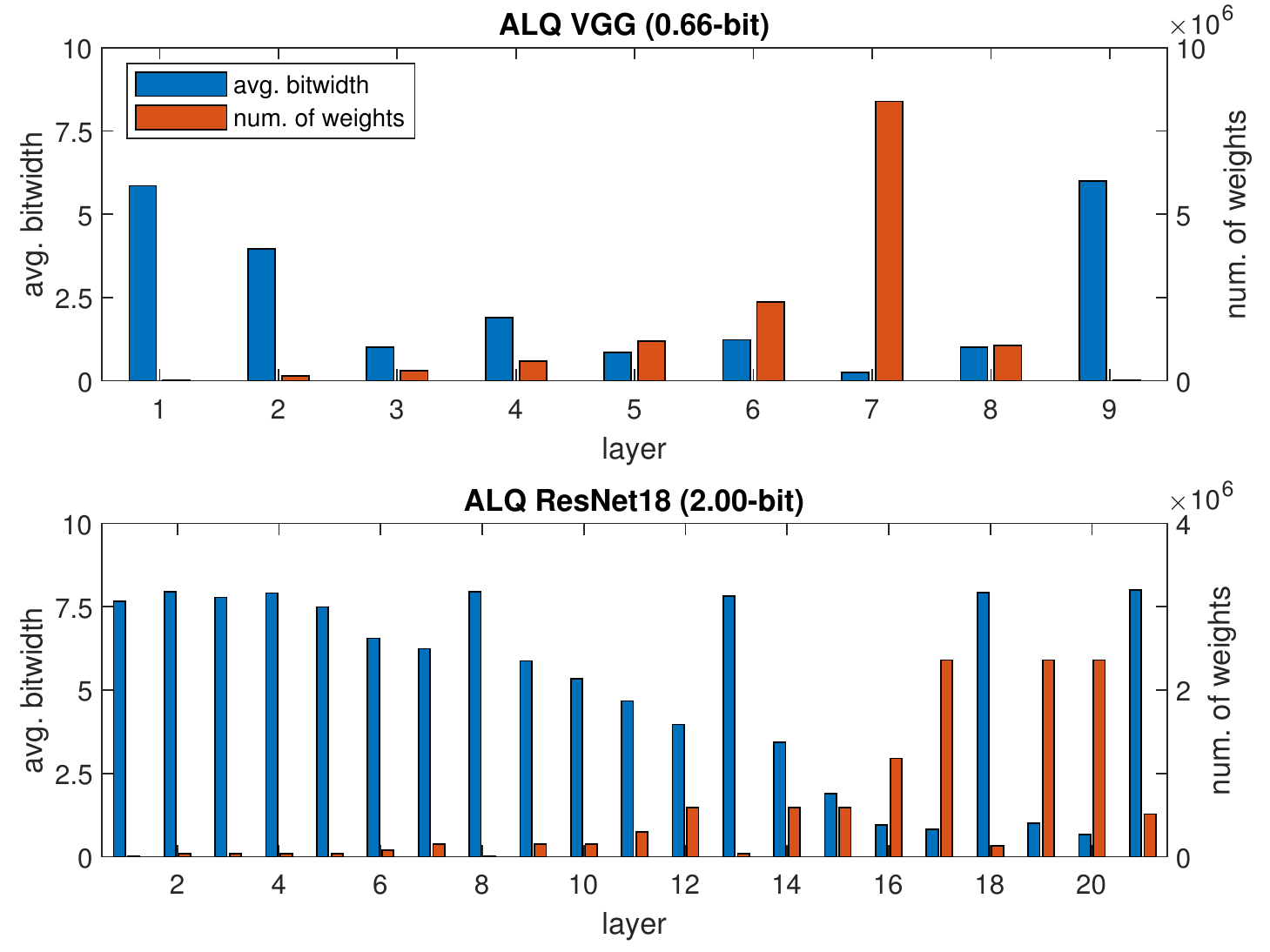}
\caption{Distribution of the average bitwidth and the number of weights across layers.}
\label{fig:adapt}
\end{figure}

It is worth noting that both the optimization step and the pruning step in \sysname follow the same metric, \ie the loss increment modeled by a quadratic function, allowing them to work in synergy. 
We replace the step of optimizing $\bm{B}_g$ in \sysname with an STE step (with the reconstruction forward, see in \secref{sec:convergence}), and keep other steps unchanged in the pipeline. 
When the VGG model is reduced to an average bitwidth of $0.66$-bit, the simple combination of an STE step with our pruning step can only reach $90.7\%$ Top-1 accuracy, which is significantly worse than \sysname's $92.0\%$.

\subsection{Comparison with States-of-the-Arts}
\label{sec:comparison}

\subsubsection{Non-structured Pruning on MNIST}
\fakeparagraph{Settings}
Since \sysname can be considered as a (structured) pruning scheme in $\bm{\alpha}$ domain, we first compare \sysname with two widely used non-structured pruning schemes: Deep Compression (DC) \cite{bib:ICLR16:Han} and ADMM-Pruning (ADMM) \cite{bib:ECCV18:Zhang2}, \ie pruning in the original $\bm{w}$ domain.
For a fair comparison, we implement a modified LeNet5 model as in \cite{bib:ICLR16:Han,bib:ECCV18:Zhang2} on MNIST dataset~\cite{bib:MNIST} and compare the Top-1 prediction accuracy and the compression rate.
Note that the storage consumption only counts the weights, since the weights take the most majority of the storage (even after quantization) in comparison to others, \eg bias, activation quantizer, batch normalization, \etc
The storage consumption of weights in \sysname includes the look-up-table for the resulting $I_g$ in each group (the same goes for the following experiments).

\begin{table}[htbp!]
\centering
\caption{Comparison with State-of-the-Art Non-structured Pruning Methods (LeNet5 on MNIST).}
\label{tab:lenet5}
\small
\begin{tabular}{ccc}
\toprule
Method                                      & Weights~(CR)                                      & Top-1                 \\ \hline
FP                                          & 1720KB~(1$\times$ )                               & 99.19\%               \\
DC~\cite{bib:ICLR16:Han}                    & 44.0KB~(39$\times$)                               & \textbf{99.26\%}      \\
ADMM~\cite{bib:ECCV18:Zhang2}               & 24.2KB~(71$\times$)                               & 99.20\%               \\
\textbf{\sysname}                           & \textbf{22.7KB}~(\textbf{76}$\bm{\times}$)        & \textbf{99.12\%}      \\ 
\bottomrule
\end{tabular}
\end{table}

\fakeparagraph{Results}
\sysname shows the highest compression rate (\textbf{76}$\bm{\times}$) while keeping acceptable Top-1 accuracy compared to the two other pruning methods (see \tabref{tab:lenet5}). 
FP stands for full precision, and the weights in the original full precision LeNet5 consume $1720$KB \cite{bib:ICLR16:Han}.
CR denotes the compression rate of storing the weights.

It is worth mentioning that both DC \cite{bib:ICLR16:Han} and ADMM \cite{bib:ECCV18:Zhang2} rely on sparse tensors, which need special libraries or hardwares for execution \cite{bib:ICLR17:Li}.
Their operands (the shared quantized values) are still floating-point.
Hence they hardly utilize bitwise operations for speedup.
In contrast, \sysname achieves a higher compression rate without sparse tensors, which is more suited for general off-the-shelf platforms.

The average bitwidth of \sysname is below $1.0$-bit ($1.0$-bit corresponds to a compression rate slightly below $32$), indicating some groups are fully removed. 
In fact, this process leads to a new network architecture containing less output channels of each layer, and thus the corresponding input channels of the next layers can be safely removed.
The original configuration $20-50-500-10$ is now $18-45-231-10$.

\subsubsection{Binary Networks on CIFAR10}
\label{sec:cifar10}
\fakeparagraph{Settings}
In this experiment, we compare the performance of \sysname with state-of-the-art binary networks \cite{bib:NIPS15:Courbariaux,bib:ECCV16:Rastegari,bib:ICLR17:Hou}.
A binary network is an MBN with the lowest bitwidth, \ie single-bit. 
Thus, the storage consumption of a binary network can be regarded as the lower bound of a (multi-bit) binary network. 
For a fair comparison, we implement a small version of VGG from~\cite{bib:ICLR15:Simonyan} on CIFAR10 dataset~\cite{cifar10}, as in many state-of-the-art binary networks~\cite{bib:NIPS15:Courbariaux, bib:ICLR17:Hou, bib:ECCV16:Rastegari}.

\begin{table}[htbp!]
  \centering
  \caption{Comparison with State-of-the-Art Binary Networks (VGG on CIFAR10).} 
  \label{tab:cifar10}
  \small
  \begin{tabular}{cccc}
  \toprule
  Method                                & $I_W$             & Weights~(CR)                          & Top-1             \\  \hline
  FP                                    & 32                & 56.09MB~(1$\times$)                   & 92.8\%            \\
  BC~\cite{bib:NIPS15:Courbariaux}      & 1                 & 1.75MB~(32$\times$)                   & 90.1\%            \\
  BWN~\cite{bib:ECCV16:Rastegari}*      & 1                 & 1.82MB~(31$\times$)                   & 90.1\%            \\
  LAB~\cite{bib:ICLR17:Hou}             & 1                 & 1.77MB~(32$\times$)                   & 89.5\%            \\
  AQ~\cite{bib:ICLR18:Khoram}           & 0.27              & 1.60MB~(35$\times$)                   & 90.9\%            \\
  \textbf{\sysname}                     & \textbf{0.66}     & \textbf{1.29MB}~(\textbf{43$\times$}) & \textbf{92.0\%}   \\
  \textbf{\sysname}                     & \textbf{0.40}     & \textbf{0.82MB}~(\textbf{68$\times$}) & \textbf{90.9\%}   \\ 
  \bottomrule
  \end{tabular}
  \begin{tablenotes}
    \item
    *: both first and last layers are unquantized.
  \end{tablenotes}
\end{table}

\fakeparagraph{Results}
\tabref{tab:cifar10} shows the performance comparison to popular binary networks.
$I_W$ stands for the quantization bitwidth for weights.
Since \sysname has an adaptive quantization bitwidth, the reported bitwidth of \sysname is an average bitwidth of all weights.
For the statistic information, we plot multiple training loss curves in Appendix~\ref{sec:vggappendix}.

\sysname allows to compress the network to under $1$-bit, which remarkably reduces the storage and computation. 
\sysname achieves the smallest weight storage and the highest accuracy compared to all weights binarization methods BC~\cite{bib:NIPS15:Courbariaux}, BWN~\cite{bib:ECCV16:Rastegari}, LAB~\cite{bib:ICLR17:Hou}.
Similar to results on LeNet5, \sysname generates a new network architecture with fewer output channels per layer, which further reduces our models in \tabref{tab:cifar10} to $1.01$MB ($0.66$-bit) or even $0.62$MB ($0.40$-bit).
The computation and the run-time memory can also decrease. 

Furthermore, we also compare with AQ~\cite{bib:ICLR18:Khoram}, the state-of-the-art adaptive fixed-point quantizer.
It assigns a different bitwidth for each parameter based on its sensitivity, and also realizes a pruning for 0-bit parameters. 
Our \sysname not only consumes less storage, but also acquires a higher accuracy than AQ~\cite{bib:ICLR18:Khoram}.
Besides, the non-standard quantization bitwidth in AQ cannot efficiently run on general hardware due to the irregularity~\cite{bib:ICLR18:Khoram}, which is not the case for \sysname.

\subsubsection{MBNs on ILSVRC12}
\label{sec:imagenet}
\fakeparagraph{Settings}
We quantize both the weights and the activations of ResNet18/34~\cite{bib:CVPR16:He} with a low bitwidth ($\leq2$-bit) on ILSVRC12 dataset~\cite{ILSVRC15}, and compare our results with state-of-the-art multi-bit networks. 
The results for the full precision version are provided by Pytorch~\cite{bib:NIPSWorkshop17:Paszke}.
We choose ResNet18, as it is a popular model on ILSVRC12 used in the previous quantization schemes. 
ResNet34 is a deeper network used more in recent quantization papers.  

\fakeparagraph{Results}
\tabref{tab:ResNet} shows that \sysname obtains the highest accuracy with the smallest network size on ResNet18/34, in comparison with other weight and weight+activation quantization approaches.
$I_W$ and $I_A$ are the quantization bitwidth for weights and activations respectively.

Several schemes (marked with *) are not able to quantize the first and last layers, since quantizing both layers as other layers will cause a huge accuracy degradation \cite{bib:ECCV18:Wan,bib:ICLR18:Mishra2}. 
It is worth noting that the first and last layers with floating-point values occupy $2.09$MB storage in ResNet18/34, which is still a significant storage consumption on such a low-bit network. 
We can simply observe this enormous difference between TWN~\cite{bib:NIPS16:Li} and LQ-Net~\cite{bib:ECCV18:Zhang} in~\tabref{tab:ResNet} for example. 
The evolved floating-point computations in both layers can hardly be accelerated with bitwise operations either. 

For reported \sysname models in~\tabref{tab:ResNet}, as several layers have already been pruned to an average bitwidth below $1.0$-bit (\eg in \figref{fig:adapt}), we add extra 50 epochs of our \textit{STE with loss-aware} in the end (see in \secref{sec:convergence}).
The final accuracy is further boosted (see the results marked with $^\mathrm{e}$). 
\sysname can quantize ResNet18/34 with 2.00-bit (across all layers) \textit{without any accuracy loss}.
To the best of our knowledge, this is the first time that the 2-bit weight-quantized ResNet18/34 can achieve the accuracy level of its full precision version, even if some prior schemes keep the first and last layers unquantized.
These results further demonstrate the high-performance of the pipeline in \sysname.

\begin{table}[tbp!]
\centering
\caption{Comparison with State-of-the-Art Multi-bit Networks (ResNet18/34 on ILSVRC12).}
\label{tab:ResNet}
\small
\begin{tabular}{cccc}
\toprule
Method                                      & $I_W$/$I_A$                 & Weights                 & Top-1                     \\ \hline
\multicolumn{4}{c}{ResNet18}                                                                                                    \\ \hdashline
FP~\cite{bib:NIPSWorkshop17:Paszke}         & 32/32                       & 46.72MB                 & 69.8\%                    \\
TWN~\cite{bib:NIPS16:Li}                    & 2/32                        & 2.97MB                  & 61.8\%                    \\
LR~\cite{bib:ICLR18:Shayer}                 & 2/32                        & 4.84MB                  & 63.5\%                    \\
LQ~\cite{bib:ECCV18:Zhang}*                 & 2/32                        & 4.91MB                  & 68.0\%                    \\
QIL~\cite{bib:CVPR19:Jung}*                 & 2/32                        & 4.88MB                  & 68.1\%                    \\
INQ~\cite{bib:ICLR17:Zhou}                  & 3/32                        & 4.38MB                  & 68.1\%                    \\
ABC~\cite{bib:NIPS17:Lin}                   & 5/32                        & 7.41MB                  & 68.3\%                    \\
\textbf{\sysname}                           & \textbf{2.00/32}            & \textbf{3.44MB}         & \textbf{68.9\%}           \\
\textbf{\sysname}$^\mathrm{e}$              & \textbf{2.00/32}            & \textbf{3.44MB}         & \textbf{70.0\%}           \\
BWN~\cite{bib:ECCV16:Rastegari}*            & 1/32                        & 3.50MB                  & 60.8\%                    \\
LR~\cite{bib:ICLR18:Shayer}*                & 1/32                        & 3.48MB                  & 59.9\%                    \\
DSQ~\cite{bib:ICCV19:Gong}*                 & 1/32                        & 3.48MB                  & 63.7\%                    \\
\textbf{\sysname}                           & \textbf{1.01/32}            & \textbf{1.77MB}         & \textbf{65.6\%}           \\
\textbf{\sysname}$^\mathrm{e}$              & \textbf{1.01/32}            & \textbf{1.77MB}         & \textbf{67.7\%}           \\
LQ~\cite{bib:ECCV18:Zhang}*                 & 2/2                         & 4.91MB                  & 64.9\%                    \\
QIL~\cite{bib:CVPR19:Jung}*                 & 2/2                         & 4.88MB                  & 65.7\%                    \\
DSQ~\cite{bib:ICCV19:Gong}*                 & 2/2                         & 4.88MB                  & 65.2\%                    \\
GroupNet~\cite{bib:CVPR19:Zhuang}*          & 4/1                         & 7.67MB                  & 66.3\%                    \\
RQ~\cite{bib:ICLR19:Louizos}                & 4/4                         & 5.93MB                  & 62.5\%                    \\
ABC~\cite{bib:NIPS17:Lin}                   & 5/5                         & 7.41MB                  & 65.0\%                    \\
\textbf{\sysname}                           & \textbf{2.00/2}             & \textbf{3.44MB}         & \textbf{66.4\%}           \\
SYQ~\cite{bib:CVPR18:Faraone}*              & 1/8                         & 3.48MB                  & 62.9\%                    \\
LQ~\cite{bib:ECCV18:Zhang}*                 & 1/2                         & 3.50MB                  & 62.6\%                    \\
PACT~\cite{bib:arXiv18:Choi}*               & 1/2                         & 3.48MB                  & 62.9\%                    \\
\textbf{\sysname}                           & \textbf{1.01/2}             & \textbf{1.77MB}         & \textbf{63.2\%}           \\ \hline
\multicolumn{4}{c}{ResNet34}                                                                                                    \\ \hdashline
FP~\cite{bib:NIPSWorkshop17:Paszke}         & 32/32                       & 87.12MB                 & 73.3\%                    \\
\textbf{\sysname}$^\mathrm{e}$              & \textbf{2.00/32}            & \textbf{6.37MB}         & \textbf{73.6\%}           \\
\textbf{\sysname}$^\mathrm{e}$              & \textbf{1.00/32}            & \textbf{3.29MB}         & \textbf{72.5\%}           \\
LQ~\cite{bib:ECCV18:Zhang}*                 & 2/2                         & 7.47MB                  & 69.8\%                    \\
QIL~\cite{bib:CVPR19:Jung}*                 & 2/2                         & 7.40MB                  & 70.6\%                    \\
DSQ~\cite{bib:ICCV19:Gong}*                 & 2/2                         & 7.40MB                  & 70.0\%                    \\
GroupNet~\cite{bib:CVPR19:Zhuang}*          & 5/1                         & 12.71MB                 & 70.5\%                    \\
ABC~\cite{bib:NIPS17:Lin}                   & 5/5                         & 13.80MB                 & 68.4\%                    \\
\textbf{\sysname}                           & \textbf{2.00/2}             & \textbf{6.37MB}         & \textbf{71.0\%}           \\ 
TBN~\cite{bib:ECCV18:Wan}*                  & 1/2                         & 4.78MB                  & 58.2\%                    \\
LQ~\cite{bib:ECCV18:Zhang}*                 & 1/2                         & 4.78MB                  & 66.6\%                    \\
\textbf{\sysname}                           & \textbf{1.00/2}             & \textbf{3.29MB}         & \textbf{67.4\%}           \\ 
\bottomrule
\end{tabular}
\begin{tablenotes}
    \item
    *: both first and last layers are unquantized.
    \item
    $^\mathrm{e}$: adding extra epochs of \textit{STE with loss-aware} in the end.
\end{tablenotes}
\end{table}

\section{Conclusion}
\label{sec:conclusion}
In this paper, we propose a novel loss-aware trained quantizer for multi-bit networks, which realizes an adaptive bitwidth for all layers (w.r.t. the loss). 
The experiments on current open datasets reveal that \sysname outperforms state-of-the-art multi-bit/binary networks in both accuracy and storage. 
Currently, we are deploying \sysname on a mobile platform to measure the inference efficiency.

\section*{Acknowledgement}
We are grateful for the anonymous reviewers and area chairs for their valuable comments and suggestions.
This research was supported in part by the Singapore Ministry of Education (MOE) Academic Research Fund (AcRF) Tier 1 grant.
Zimu Zhou is the corresponding author.

\newpage 

{\small
\bibliographystyle{ieee_fullname}
\bibliography{cites}
}

\newpage
\section*{Appendix}

\appendix

\section{\sysname Initialization}
\label{sec:initialization}

\subsection{Initialization Algorithm}
\label{sec:initalgorithm}

We adapt the network sketching in~\cite{bib:CVPR17:Guo}, and propose a structured sketching algorithm below for \sysname initialization (see Alg.~\ref{alg:SS})\footnote{Circled operation in Alg.~\ref{alg:SS} means elementwise operations.}. 
Here, the subscript of the layer index $l$ is reintroduced for a layerwise elaboration in the pseudocode.
This algorithm partitions the pretrained full precision weights $\bm{w}_l$ of the $l^{\text{th}}$ layer into $G_l$ groups with the structures mentioned in \ref{sec:group}. 
The vectorized weights $\bm{w}_{l,g}$ of each group are quantized with $I_{l,g}$ linear independent binary bases (\ie column vectors in $\bm{B}_{l,g}$) and corresponding coordinates $\bm{\alpha}_{l,g}$ to minimize the reconstruction error. 
This algorithm initializes the matrix of binary bases $\bm{B}_{l,g}$, the vector of floating-point coordinates $\bm{\alpha}_{l,g}$, and the scalar of integer bitwidth $I_{l,g}$ in each group across layers.
The initial reconstruction error is upper bounded by a threshold $\sigma$. 
In addition, a maximum bitwidth of each group is defined as $\mathrm{I_{max}}$.
Both of these two parameters determine the initial bitwidth $I_{l,g}$.

\begin{algorithm}[!htbp]
\caption{Structured Sketching of Weights}\label{alg:SS}
\KwIn{$\{\bm{w}_l\}_{l=1}^{L}$, $\{G_l\}_{l=1}^{L}$, $\mathrm{I_{max}}$, $\sigma$}
\KwOut{$\{\{\bm{\alpha}_{l,g},\bm{B}_{l,g}, I_{l,g}\}_{g=1}^{G_l}\}_{l=1}^{L}$}
\For {$l\leftarrow 1$ \KwTo $L$} {
  \For {$g \leftarrow 1$ \KwTo $G_l$} {
    \textbf{Fetch and vectorize} $\bm{w}_{l,g}$ \textbf{from} $\bm{w}_l$ \;
    \textbf{Initialize} $\bm{\epsilon} = \bm{w}_{l,g}$, $i=0$ \;
    $\bm{B}_{l,g} = [~]$ \;
    \While{$\|\bm{\epsilon}\oslash\bm{w}_{l,g}\|_2^2>\sigma$ \textbf{\texttt{\textup{and}}} $i<\mathrm{I_{max}}$} 
    {
      $i = i+1$\;
      $\bm{\beta}_{i} = \mathrm{sign}(\bm{\epsilon})$\;
      $\bm{B}_{l,g} = [\bm{B}_{l,g}, \bm{\beta}_{i}]$\;
      \tcc{Find the optimal point spanned by $\bm{B}_{l,g}$}
      $\bm{\alpha}_{l,g} = (\bm{B}_{l,g}^\mathrm{T}\bm{B}_{l,g})^{-1}\bm{B}_{l,g}^\mathrm{T}\bm{w}_{l,g}$ \;
       \tcc{Update the residual reconstruction error}
      $\bm{\epsilon} = \bm{w}_{l,g}-\bm{B}_{l,g}\bm{\alpha}_{l,g}$ \;
   }
    $I_{l,g}=i$\;
  }
}
\end{algorithm}

\begin{theorem}
The column vectors in $\bm{B}_{l,g}$ are linear independent.
\end{theorem}

\begin{proof}
The instruction $\bm{\alpha}_{l,g} = (\bm{B}_{l,g}^\mathrm{T}\bm{B}_{l,g})^{-1}\bm{B}_{l,g}^\mathrm{T}\bm{w}_{l,g}$ ensures $\bm{\alpha}_{l,g}$ is the optimal point in $\mathrm{span}(\bm{B}_{l,g})$ regarding the least square reconstruction error $\bm{\epsilon}$. 
Thus, $\bm{\epsilon}$ is orthogonal to $\mathrm{span}(\bm{B}_{l,g})$. 
The new basis is computed from the next iteration by $\bm{\beta}_{i}= \mathrm{sign}(\bm{\epsilon})$. 
Since $\mathrm{sign}(\bm{\epsilon})\bullet\bm{\epsilon}>0, \forall\bm{\epsilon}\ne\bm{0}$, we have $\bm{\beta}_{i}\notin \mathrm{span}(\bm{B}_{l,g})$. 
Thus, the iteratively generated column vectors in $\bm{B}_{l,g}$ are linear independent.
This also means the square matrix of $\bm{B}_{l,g}^\mathrm{T}\bm{B}_{l,g}$ is invertible.
\end{proof}

\subsection{Experiments on Group Size}
\label{sec:group}

Researchers propose different structured quantization in order to exploit the redundancy and the tolerance in the different structures.
Certainly, the weights in one layer can be arbitrarily selected to gather a group.
Considering the extra indexing cost, in general, the weights are sliced along the tensor dimensions and uniformly grouped. 

According to~\cite{bib:CVPR17:Guo}, the squared reconstruction error of a single group decays with~\equref{eq:reconstructionErrorDecay}, where $\lambda\ge0$.
\begin{equation}
\|\bm{\epsilon}\|_2^2 \le \|\bm{w}_{g}\|_2^2 (1-\frac{1}{n-\lambda})^{I_g}
\label{eq:reconstructionErrorDecay}
\end{equation}
If full precision values are stored in floating-point datatype, \ie $32$-bit, the storage compression rate in one layer can be written as,
\begin{equation}
r_s = \frac{N\times32}{I\times N+I\times32\times \frac{N}{n}}
\label{eq:r_s}
\end{equation}
where $N$ is the total number of weights in one layer; $n$ is the number of weights in each group, \ie $n = N/G$; $I$ is the average bitwidth, $I = \frac{1}{G}\sum_{g = 1}^G I_g$ .

We analyse the trade-off between the reconstruction error and the storage compression rate of different group size $n$.
We choose the pretrained AlexNet~\cite{bib:NIPS12:Krizhevsky} and VGG-16~\cite{bib:ICLR15:Simonyan}, and plot the curves of the average (per weight) reconstruction error related to the storage compression rate of each layer under different sliced structures.
We also randomly shuffle the weights in each layer, then partition them into groups with different sizes.
We select one example plot which comes from the last convolution layer ($256\times256\times3\times3$) of AlexNet~\cite{bib:NIPS12:Krizhevsky} (see~\figref{fig:conv_alexnet}).
The pretrained full precision weights are provided by Pytorch~\cite{bib:NIPSWorkshop17:Paszke}. 
\begin{figure}[htbp!]
\centering
\includegraphics[width=0.45\textwidth]{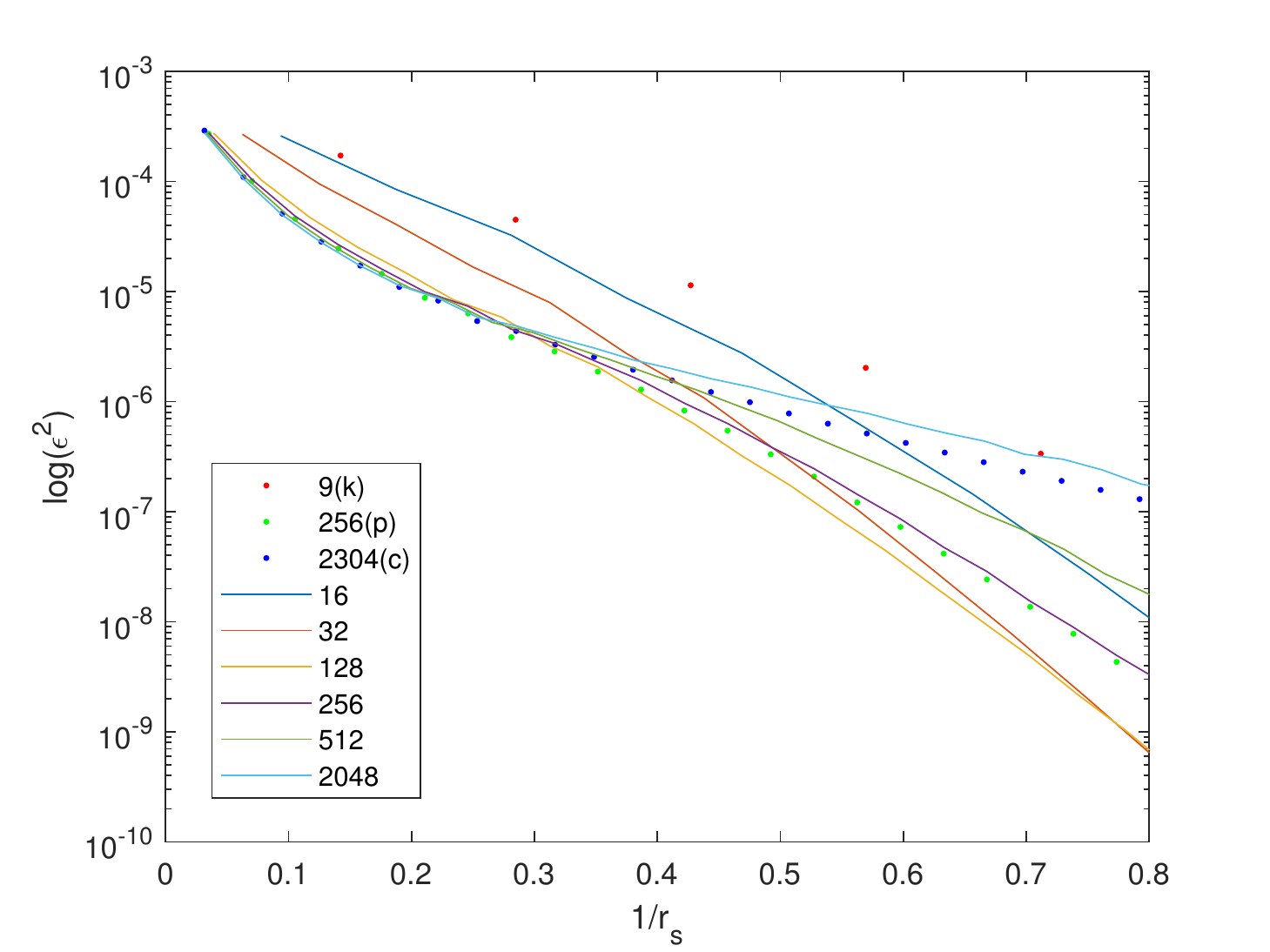}
\caption{The curves about the logarithmic L2-norm of the average reconstruction error $\mathrm{log}(\|\bm{\epsilon}\|_2^2)$ related to the reciprocal of the storage compression rate $1/r_s$ (from the last convolution layer of AlexNet). The legend demonstrates the corresponding group sizes. 'k' stands for kernel-wise; 'p' stands for pixel-wise; 'c' stands for channel-wise. }
\label{fig:conv_alexnet}
\end{figure}

We have found that there is not a significant difference between random groups and sliced groups (along original tensor dimensions).
Only the group size influences the trade-off.
We argue the reason is that one layer always contains thousands of groups, such that the points presented by these groups are roughly scattered in the $n$-dim space.
Furthermore, regarding the deployment on a 32-bit general microprocessor, the group size should be larger than 32 for efficient computation.
In short, a group size from $32$ to $512$ achieves relatively good trade-off between the reconstruction error and the storage compression. 

These above demonstrated three structures in~\figref{fig:conv_alexnet} do not involve the cross convolutional filters' computation, which leads to less run-time memory than other structures.
Accordingly, for a convolution layer, grouping in channel-wise ($\bm{w}_{c,:,:,:}$), kernel-wise ($\bm{w}_{c,d,:,:}$), and pixel-wise ($\bm{w}_{c,:,h,w}$) are appropriate.
Channel-wise $\bm{w}_{c,:}$ and subchannel-wise $\bm{w}_{c,d:d+n}$ grouping are suited for a fully connected layer.
The most frequently used structures for current popular network are pixel-wise (convolution layers) and (sub)channel-wise (fully connected layers), which exactly coincide the bit-packing approach in \cite{bib:ICLR18:Pedersoli}, and could result in a more efficient deployment.
Since many network architectures choose an integer multiple of 32 as the number of output channels in each layer, pixel-wise and (sub)channel-wise are also efficient for the current storage format in 32-bit microprocessors, \ie in 4 Bytes (32-bit integer).

\section{Pseudocode and Complexity Analysis}

\subsection{Pruning in $\bm{\alpha}$ Domain}
\label{sec:pseudocodepruning}

In each execution of Step 1 (\secref{sec:pruning}), $30\%$ of $\alpha_i$'s are pruned. 
Iterative pruning is realized in mini-batch (general 1 epoch in total). 
Due to the high complexity of sorting all $f_{\bm{\alpha},i}$, sorting is firstly executed in each layer, and the top-$k\%$ $f_{\bm{\alpha}_l,i}$ of the $l^{\text{th}}$ layer are selected to resort again for pruning.
$k$ is generally small, \eg $1$ or $0.5$, which ensures that the pruned $\alpha_i$'s in one iteration do not come from a single layer.
Again, $\bm{\alpha}_l$ is vectorized $\{\bm{\alpha}_{l,g}\}_{g=1}^{G_l}$; $\bm{B}_l$ is concatenated $\{\bm{B}_{l,g}\}_{g=1}^{G_l}$ in the $l^{\text{th}}$ layer.
There are $n_l$ weights in each group, and $G_l$ groups in the $l^{\text{th}}$ layer.

The number of total layers is usually smaller than $100$, thus, the sorting complexity mainly depends on the sorting in the layer, which has the largest $\mathrm{card}(\bm{\alpha}_l)$.
The number of the sorted element $f_{\bm{\alpha}_l,i}$, \ie $\mathrm{card}(\bm{\alpha}_l)$, is usually smaller than an order of $10^4$ for a general network in \sysname.

The pruning step in~\secref{sec:pruning} is demonstrated in Alg.~\ref{alg:pruning}. 
Here, assume that there are altogether $T$ times pruning iterations in each execution of Step 1; the total number of $\alpha_i$'s across all layers is $M_0$ before pruning, \ie
\begin{equation}
M_0 = \underset{l}{\sum}{\underset{g}{\sum}{\mathrm{card}(\bm{\alpha}_{l,g})}}
\label{eq:m0}
\end{equation}
and the desired total number of $\alpha_i$'s after pruning is $M_T$.

\begin{algorithm}[htbp!]
\caption{Pruning in $\alpha$ Domain}\label{alg:pruning}
\KwIn{$T$, $M_T$, $k$, $\{\{\bm{\alpha}_{l,g},\bm{B}_{l,g}, I_{l,g}\}_{g=1}^{G_l}\}_{l=1}^L$, Training Data}
\KwOut{$\{\{\bm{\alpha}_{l,g},\bm{B}_{l,g}, I_{l,g}\}_{g=1}^{G_l}\}_{l=1}^L$}
\textbf{Compute} $M_0$ \textbf{with} \equref{eq:m0} \;
\textbf{Determine the pruning number at each iteration} $M_p = \mathrm{round}(\frac{M_0-M_T}{T})$ \;
\For {$t \leftarrow 1$ \KwTo $T$} {
  \For {$l\leftarrow 1$ \KwTo $L$} {
    \textbf{Update} $\bm{\hat{w}}_{l,g}^t = \bm{B}_{l,g}^t\bm{\alpha}_{l,g}^t$  \;
    \textbf{Forward propagate convolution} \;
  }
  \textbf{Compute the loss} $\ell^t$ \;
  \For {$l\leftarrow L$ \KwTo $1$} {
    \textbf{Backward propagate gradient} $\frac{\partial\ell^t}{\partial\bm{\hat{w}}_{l,g}^t}$ \;
    \textbf{Compute} $\frac{\partial\ell^t}{\partial\bm{\alpha}_{l,g}^t}$ \textbf{with} \equref{eq:gradientspruning} \;
    \textbf{Update momentums of AMSGrad in $\bm{\alpha}$ domain} \; 
    \For {$\alpha_{l,i}^t$ \textup{\textbf{in}} $\bm{\alpha}_l^t$} {
      \textbf{Compute} $f_{\bm{\alpha}_l,i}^t$ \textbf{with} \equref{eq:taylorpruning} \;
    }
    \textbf{Sort and select Top-}$k\%$ $f_{\bm{\alpha}_l,i}^t$ \textbf{in ascending order} \;
  }
  \textbf{Resort the selected} $\{f_{\bm{\alpha}_l,i}^t\}_{l=1}^{L}$ \textbf{in ascending order} \;
  \textbf{Remove Top-}$M_p$ $\alpha_{l,i}^t$ \textbf{and their binary bases} \;
  \textbf{Update} $\{\{\bm{\alpha}_{l,g}^{t+1},\bm{B}_{l,g}^{t+1}, I_{l,g}^{t+1}\}_{g=1}^{G_l}\}_{l=1}^L$ \;
}
\end{algorithm}

\subsection{Optimizing Binary Bases and Coordinates}
\label{sec:appendixupdating}
Step 2 is also executed in batch training.
In Step 2 (\secref{sec:updating}), $10^{-3}$ is used as the learning rate in optimizing $\bm{B}_g$, and gradually decays in each epoch; the learning rate is set to $10^{-5}$ in optimizing $\bm{\alpha}_g$, and also gradually decays in each epoch. 

\subsubsection{Optimizing $\bm{B}_g$ with Speedup}
\label{sec:pseudocodeB}
The extra complexity related to the original AMSGrad mainly comes from two parts, \equref{eq:rowupdating} and \equref{eq:alphaIncr}.
\equref{eq:rowupdating} is also the most resource-hungry step of the whole pipeline, since it requires an exhaustive search.
For each group, \equref{eq:rowupdating} takes both time and storage complexities of $O(n2^{I_g})$, and in general $n>>I_g\geq1$.
Since $\bm{H}^q$ is a diagonal matrix, most of matrix-matrix multiplication in \equref{eq:alphaIncr} is avoided through matrix-vector multiplication and matrix-diagonalmatrix multiplication.
Thus, the time complexity trims down to $O(nI_g+nI_g^2+I_g^3+nI_g+n+n+nI_g+I_g^2) \doteq O(n(I_g^2+3I_g+2))$.
In our settings, optimizing $\bm{B}_g$ with speedup usually takes around twice as long as optimizing $\bm{\alpha}_g$ (\ie the original AMSGrad step).

Optimizing $\bm{B}_g$ with speedup (\secref{sec:updating}) is presented in Alg.~\ref{alg:optimizingbases}. 
Assume that there are altogether $Q$ iterations. 
It is worth noting that the bitwidth $I_{l,g}$ does not change in this step; only the binary bases $\bm{B}_{l,g}$ and the coordinates $\bm{\alpha}_{l,g}$ are updated over $Q$ iterations.  

\begin{algorithm}[htbp!]
\caption{Optimizing $\bm{B}_g$ with Speedup}\label{alg:optimizingbases}
\KwIn{$Q$, $\{\{\bm{\alpha}_{l,g},\bm{B}_{l,g}, I_{l,g}\}_{g=1}^{G_l}\}_{l=1}^L$, Training Data}
\KwOut{$\{\{\bm{\alpha}_{l,g},\bm{B}_{l,g}, I_{l,g}\}_{g=1}^{G_l}\}_{l=1}^L$}
\For {$q \leftarrow 1$ \KwTo $Q$} {
  \For {$l\leftarrow 1$ \KwTo $L$} {
    \textbf{Update} $\bm{\hat{w}}_{l,g}^q = \bm{B}_{l,g}^q\bm{\alpha}_{l,g}^q$  \;
    \textbf{Forward propagate convolution} \;
  }
  \textbf{Compute the loss} $\ell^q$ \;
  \For {$l\leftarrow L$ \KwTo $1$} {
    \textbf{Backward propagate gradient} $\frac{\partial\ell^q}{\partial\bm{\hat{w}}_{l,g}^q}$ \;
    \textbf{Update momentums of AMSGrad} \; 
    \For {$g \leftarrow 1$ \KwTo $G_l$} {
      \textbf{Compute all values of} \equref{eq:comb} \;
      \For {$j \leftarrow 1$ \KwTo $n_l$} {
        \textbf{Update} $\bm{B}_{l,g,j}^{q+1}$ \textbf{according to the nearest value} (\textbf{see} \equref{eq:rowupdating}) \;
      }
      \textbf{Update} $\bm{\alpha}_{l,g}^{q+1}$ \textbf{with} \equref{eq:alphaIncr} \;
    }
  }
}
\end{algorithm}

\subsubsection{Optimizing $\bm{\alpha}_g$}
\label{sec:pseudocodealpha}

Since $\bm{\alpha}_g$ is floating-point value, the complexity of optimizing $\bm{\alpha}_g$ is the same as the conventional optimization step, (see Alg.~\ref{alg:optimizingcoordinates}).
Assume that there are altogether $P$ iterations. 
It is worth noting that both the bitwidth $I_{l,g}$ and the binary bases $\bm{B}_{l,g}$ do not change in this step; only the coordinates $\bm{\alpha}_{l,g}$ are updated over $P$ iterations. 

\begin{algorithm}[htbp!]
\caption{Optimizing $\bm{\alpha}_g$}\label{alg:optimizingcoordinates}
\KwIn{$P$, $\{\{\bm{\alpha}_{l,g},\bm{B}_{l,g}, I_{l,g}\}_{g=1}^{G_l}\}_{l=1}^L$, Training Data}
\KwOut{$\{\{\bm{\alpha}_{l,g},\bm{B}_{l,g}, I_{l,g}\}_{g=1}^{G_l}\}_{l=1}^L$}
\For {$p \leftarrow 1$ \KwTo $P$} {
  \For {$l\leftarrow 1$ \KwTo $L$} {
    \textbf{Update} $\bm{\hat{w}}_{l,g}^p = \bm{B}_{l,g}\bm{\alpha}_{l,g}^p$  \;
    \textbf{Forward propagate convolution} \;
  }
  \textbf{Compute the loss} $\ell^p$ \;
  \For {$l\leftarrow L$ \KwTo $1$} {
    \textbf{Backward propagate gradient} $\frac{\partial\ell^p}{\partial\bm{\hat{w}}_{l,g}^p}$ \;
    \textbf{Compute} $\frac{\partial\ell^p}{\partial\bm{\alpha}_{l,g}^p}$ \textbf{with} \equref{eq:gradientspruning} \;
    \textbf{Update momentums of AMSGrad in $\bm{\alpha}$ domain} \;
    \For {$g \leftarrow 1$ \KwTo $G_l$} {
      \textbf{Update} $\bm{\alpha}_{l,g}^{p+1}$ \textbf{with} \equref{eq:optimizingalpha} \;
    }
  }
}
\end{algorithm}

\subsection{Whole Pipeline of \sysname}
\label{sec:pipeline}
The whole pipeline of \sysname is demonstrated in Alg.~\ref{alg:pipeline}.

For the initialization, the pretrained full precision weights (model) $\{\bm{w}_l\}_{l=1}^{L}$ is required. 
Then, we need to specify the structure used in each layer, \ie the manner of grouping (for short marked with $\{G_l\}_{l=1}^{L}$).
In addition, a maximum bitwidth $\mathrm{I_{max}}$ and a threshold $\sigma$ for the residual reconstruction error also need to be determined (see more details in \ref{sec:initialization}).
After initialization, we might need to retrain the model with several epochs of \ref{sec:appendixupdating} to recover the accuracy degradation caused by the initialization. 

Then, we need to determine the number of outer iterations $R$, \ie how many times the pruning step (Step 1) is executed.
A pruning schedule $\{M^r\}_{r=1}^{R}$ is also required. 
$M^r$ determines the total number of remained $\alpha_i$'s (across all layers) after the $r^{\mathrm{th}}$ pruning step, which is also taken as the input $M_T$ in Alg.~\ref{alg:pruning}. 
For example, we can build this schedule by pruning $30\%$ of $\alpha_i$'s during each execution of Step 1, as,
\begin{equation}
M^{r+1} = M^{r}\times(1-0.3)
\label{eq:mr}
\end{equation}
with $r\in\{0,1,2,~...,~R-1\}$. $M^0$ represents the total number of $\alpha_i$'s (across all layers) after initialization.

For Step 1 (Pruning in $\bm{\alpha}$ Domain), other individual inputs include the total number of iterations $T$, and the selected percentages $k$ for sorting (see Alg.~\ref{alg:pruning}).
For Step 2 (Optimizing Binary Bases and Coordinates), the individual inputs includes the total number of iterations $Q$ in optimizing $\bm{B}_g$ (see Alg.~\ref{alg:optimizingbases}), and the total number of iterations $P$ in optimizing $\bm{\alpha}_g$ (see Alg.~\ref{alg:optimizingcoordinates}).

\begin{algorithm}[htbp!]
\caption{Adaptive Loss-aware Quantization}\label{alg:pipeline}
\KwIn{Pretrained FP Weights $\{\bm{w}_l\}_{l=1}^{L},~~~~~~~~~~~~$ Structures $\{G_l\}_{l=1}^{L}$, $\mathrm{I_{max}}$, $\sigma,~~~~~~~~~~~~~~~~~~~~~~~~~~~~$ $T$, Pruning Schedule $\{M^{r}\}_{r=1}^{R}$, $k,~~~~~~~~~~~~~~~~~~~~$ $P$, $Q$, $R$, Training Data}
\KwOut{$\{\{\bm{\alpha}_{l,g},\bm{B}_{l,g}, I_{l,g}\}_{g=1}^{G_l}\}_{l=1}^L$}
\tcc{Initialization: }
\textbf{Initialize} $\{\{\bm{\alpha}_{l,g},\bm{B}_{l,g}, I_{l,g}\}_{g=1}^{G_l}\}_{l=1}^L$ \textbf{with Alg.~\ref{alg:SS}} \;  
\For {$r \leftarrow 1$ \KwTo $R$} {
  \tcc{Step 1: }
  \textbf{Assign} $M^r$ \textbf{to the input} $M_T$ \textbf{of Alg.~\ref{alg:pruning}} \;
  \textbf{Prune in} $\bm{\alpha}$ \textbf{domain with Alg.~\ref{alg:pruning}} \;
  \tcc{Step 2: }
  \textbf{Optimize binary bases} \textbf{with Alg.~\ref{alg:optimizingbases}} \;
  \textbf{Optimize coordinates} \textbf{with Alg.~\ref{alg:optimizingcoordinates}} \;
}
\end{algorithm}

\subsection{STE with Loss-aware}
\label{sec:ste}
In this section, we provide the details of the proposed \textit{STE with loss-aware} optimizer.
The training scheme of \textit{STE with loss-aware} is similar as Optimizing $\bm{B}_g$ with Speedup (see~\ref{sec:pseudocodeB}), except that it maintains the full precision weights $\bm{w}_g$.
See the pseudocode of \textit{STE with loss-aware} in Alg.~\ref{alg:ste}.

For the layer $l$, the quantized weights $\bm{\hat{w}}_g$ is used during forward propagation.
During backward propagation, the loss gradients to the full precision weights $\frac{\partial\ell}{\partial\bm{w}_{g}}$ are directly approximated with $\frac{\partial\ell}{\partial {\bm{\hat{w}}_{g}}}$, \ie via STE in the $q^{\text{th}}$ iteration as, 
\begin{equation}
\frac{\partial\ell^q}{\partial\bm{w}_g^q}=\frac{\partial\ell^q}{\partial {\bm{\hat{w}}_g}^q}
\end{equation}
Then the first and second momentums in AMSGrad are updated with $\frac{\partial\ell^q}{\partial\bm{w}_{g}^q}$.
Accordingly, the loss increment around $\bm{w}_g^q$ is modeled as,
\begin{equation}
f_{ste}^q=(\bm{g}^q)^{\mathrm{T}}(\bm{w}_g-\bm{w}_g^q)+\frac{1}{2} (\bm{w}_g-\bm{w}_g^q)^{\mathrm{T}} \bm{H}^q (\bm{w}_g-\bm{w}_g^q)
\label{eq:steB}
\end{equation}
Since $\bm{w}_g$ is full precision, $\bm{w}_g^{q+1}$ can be directly obtained through the above AMSGrad step without projection updating,
\begin{equation}
\bm{w}_g^{q+1} = \bm{w}_g^q-({\bm{H}^q})^{-1}\bm{g}^q = \bm{w}_g^q-a^q\bm{m}^q/\sqrt{\bm{\hat{v}}^q}
\label{eq:steUpdateW}
\end{equation}
For more details about the notations, please refer to \secref{subsec:overview}.
Similarly, the loss increment caused by $\bm{B}_g$ (see \equref{eq:amsgradB} and \equref{eq:amsgradB2}) is formulated as,
\begin{equation}
\begin{split}
f_{ste,\bm{B}}^q=&~(\bm{g}^q)^{\mathrm{T}}(\bm{B}_g\bm{\alpha}_g^{q}-\bm{w}_g^q)+ \\
    &~\frac{1}{2} (\bm{B}_g\bm{\alpha}_g^{q}-\bm{w}_g^q)^{\mathrm{T}} \bm{H}^q (\bm{B}_g\bm{\alpha}_g^{q}-\bm{w}_g^q)
\label{eq:steB2}
\end{split}
\end{equation}
Thus, the $j^{\text{th}}$ row in $\bm{B}_g^{q+1}$ is updated by,  
\begin{equation}
\bm{B}_{g,j}^{q+1} = \underset{\bm{B}_{g,j}}{\mathrm{argmin}}~\|\bm{B}_{g,j}\bm{\alpha}_{g}^q-(w_{g,j}^q-g^q_j/H_{jj}^q)\|
\label{eq:steRowUpdating}
\end{equation}
In addition, the speedup step (see \equref{eq:incrw} and \equref{eq:alphaIncr}) is,
\begin{equation}
\begin{split}
\bm{\alpha}_{g}^{q+1} = -&((\bm{B}_{g}^{q+1})^{\mathrm{T}} \bm{H}^q \bm{B}_{g}^{q+1})^{-1}\times \\
                        &~((\bm{B}_{g}^{q+1})^{\mathrm{T}}(\bm{g}^q-\bm{H}^q\bm{w}^q_{g}))
\label{eq:steAlphaIncr}
\end{split}
\end{equation}
So far, the quantized weights are updated in a loss-aware manner as,
\begin{equation}
\bm{\hat{w}}_{g}^{q+1} = \bm{B}_{g}^{q+1}\bm{\alpha}_{g}^{q+1}
\end{equation}

\begin{algorithm}[htbp!]
\caption{STE with Loss-aware}\label{alg:ste}
\KwIn{$Q$, $\{\{\bm{\alpha}_{l,g},\bm{B}_{l,g}, I_{l,g}\}_{g=1}^{G_l}\}_{l=1}^L$, Training Data}
\KwOut{$\{\{\bm{\alpha}_{l,g},\bm{B}_{l,g}, I_{l,g}\}_{g=1}^{G_l}\}_{l=1}^L$}
\For {$q \leftarrow 1$ \KwTo $Q$} {
  \For {$l\leftarrow 1$ \KwTo $L$} {
    \textbf{Update} $\bm{\hat{w}}_{l,g}^q = \bm{B}_{l,g}^q\bm{\alpha}_{l,g}^q$  \;
    \textbf{Forward propagate convolution} \;
  }
  \textbf{Compute the loss} $\ell^q$ \;
  \For {$l\leftarrow L$ \KwTo $1$} {
    \textbf{Backward propagate gradient} $\frac{\partial\ell^q}{\partial\bm{\hat{w}}_{l,g}^q}$ \;
    \textbf{Directly approximate} $\frac{\partial\ell^q}{\partial\bm{w}_{l,g}^q}$ \textbf{with} $\frac{\partial\ell^q}{\partial {\bm{\hat{w}}_{l,g}}^q}$ \;
    \textbf{Update momentums of AMSGrad} \; 
    \For {$g \leftarrow 1$ \KwTo $G_l$} {
      \textbf{Update} $\bm{w}_{l,g}^{q+1}$ \textbf{with} \equref{eq:steUpdateW} \;
      \textbf{Compute all values of} \equref{eq:comb} \;
      \For {$j \leftarrow 1$ \KwTo $n_l$} {
        \textbf{Update} $\bm{B}_{l,g,j}^{q+1}$ \textbf{according to the nearest value} (\textbf{see} \equref{eq:steRowUpdating}) \;
      }
      \textbf{Update} $\bm{\alpha}_{l,g}^{q+1}$ \textbf{with} \equref{eq:steAlphaIncr} \;
    }
  }
}
\end{algorithm}

\section{Implementation Details}
\label{sec:implmentationDetails}

\subsection{LeNet5 on MNIST}
\label{sec:lenet5appendix}
The MNIST dataset~\cite{bib:MNIST} consists of $28\times28$ gray scale images from 10 digit classes. 
We use 50000 samples in the training set for training, the rest 10000 for validation, and the 10000 samples in the test set for testing. 
The test accuracy is reported, when the validation dataset has the highest top-1 accuracy. 
We use a mini-batch with size of 128. 
The used LeNet5 is a modified version of~\cite{bib:PIEEE98:LeCun}. 
For data preprocessing, we use the official example provided by~\cite{torchLeNet5}.
We use the default hyperparameters proposed in~\cite{torchLeNet5} to train LeNet5 for 100 epochs as the baseline of full precision version.

The network architecture is presented as,\\
20C5 - MP2 - 50C5 - MP2 - 500FC - 10SVM.

The structures of each layer chosen for \sysname are \textit{kernel-wise, kernel-wise, subchannel-wise(2), channel-wise} respectively. 
The \textit{subchannel-wise(2)} structure means all input channels are sliced into two groups with the same size, \ie the group size here is $800/2=400$.
After each pruning, the network is retrained to recover the accuracy degradation with 20 epochs of optimizing $\bm{B}_g$ and 10 epochs of optimizing $\bm{\alpha}_g$. 
The pruning ratio is 80\%, and 4 times pruning (Step 1) are executed after initialization in the reported experiment (\tabref{tab:lenet5}). 
In the end, \ie after the retraining of the last pruning step, we add another 50 epochs of optimizing steps (\secref{sec:updating}) to further increase the final accuracy (also applied in the following experiments of VGG and ResNet18/34).

\sysname can fast converge in the training. 
However, we observe that even after the convergence, the accuracy still continue increasing slowly along the training, which is similar as the behavior of STE-based optimizer. 
During the retraining after each pruning step, as long as the training loss is (almost) converged with a few epochs, we can further proceed the next pruning step. 
We have tested that the final accuracy level is approximately the same whether we add plenty of epochs each time to slowly recover the accuracy to the original level or not.
Thus, we choose a fixed modest number of retraining epochs after each pruning step to save the overall training time.
In fact, this benefits from the feature of \sysname, which leverages the true gradient (w.r.t. the loss) to result a fast and stable convergence.
The final added plenty of training epochs aim to further slowly regain the final accuracy level, and we use a gradually decayed learning rate in this process, \eg $10^{-4}$ decays with 0.98 in each epoch.

\subsection{VGG on CIFAR10}
\label{sec:vggappendix}
The CIFAR-10 dataset~\cite{cifar10} consists of 60000 $32\times32$ color images in 10 object classes. 
We use 45000 samples in the training set for training, the rest 5000 for validation, and the 10000 samples in the test set for testing. 
We use a mini-batch with size of 128. 
The used VGG net is a modified version of the original VGG~\cite{bib:ICLR15:Simonyan}.
For data preprocessing, we use the setting provided by~\cite{torchCIFAR10}.
We use the default Adam optimizer provided by Pytorch to train full precision parameters for 100 epochs as the baseline of the full precision version. 
The initial learning rate is $0.01$, and it decays with 0.2 every $30$ epochs.

The network architecture is presented as,\\
2$\times$128C3 - MP2 - 2$\times$256C3 - MP2 - 2$\times$512C3 - MP2 - 2$\times$1024FC - 10SVM.

The structures of each layer chosen for \sysname are \textit{channel-wise, pixel-wise, pixel-wise, pixel-wise, pixel-wise, pixel-wise, subchannel-wise(16), subchannel-wise(2), subchannel-wise(2)} respectively.
After each pruning, the network is retrained to recover the accuracy degradation with 20 epochs of optimizing $\bm{B}_g$ and 10 epochs of optimizing $\bm{\alpha}_g$. 
The pruning ratio is 40\%, and 5/6 times pruning (Step 1) are executed after initialization in the reported experiment (\tabref{tab:cifar10}). 

In order to demonstrate the convergence of \sysname statistically, we plot the train loss curves (the mean of cross-entropy loss) of quantizing VGG on CIFAR10 with \sysname in 5 successive trials (see~\figref{fig:5loss}). We also plot one of them with detailed training steps of \sysname (see~\figref{fig:1loss}). 

\begin{figure}[htbp!]
     \centering
     \subfloat[]{\label{fig:5loss}
	    \includegraphics[width=0.45\linewidth]{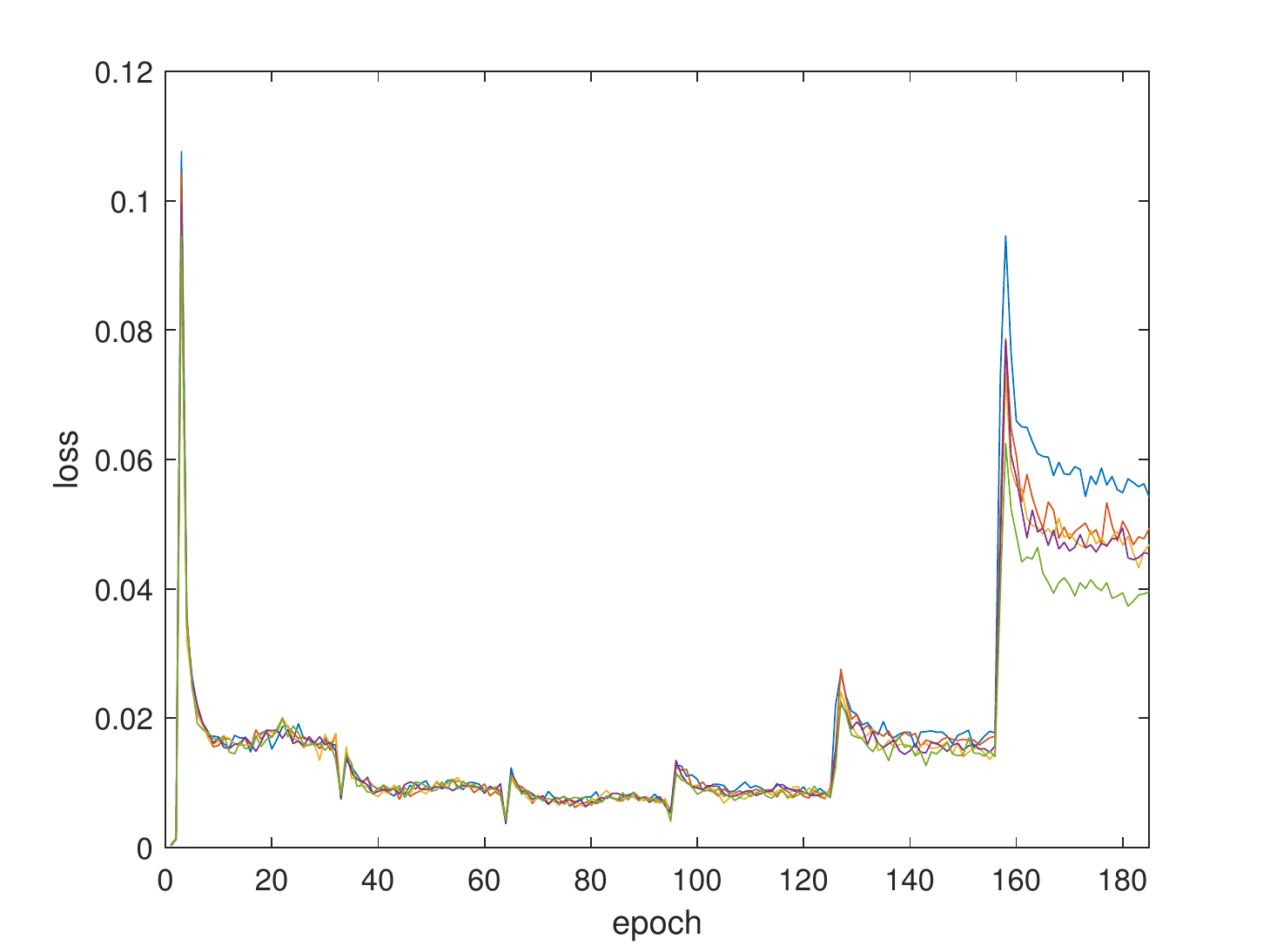}}
	 \subfloat[]{\label{fig:1loss}
	    \includegraphics[width=0.45\linewidth]{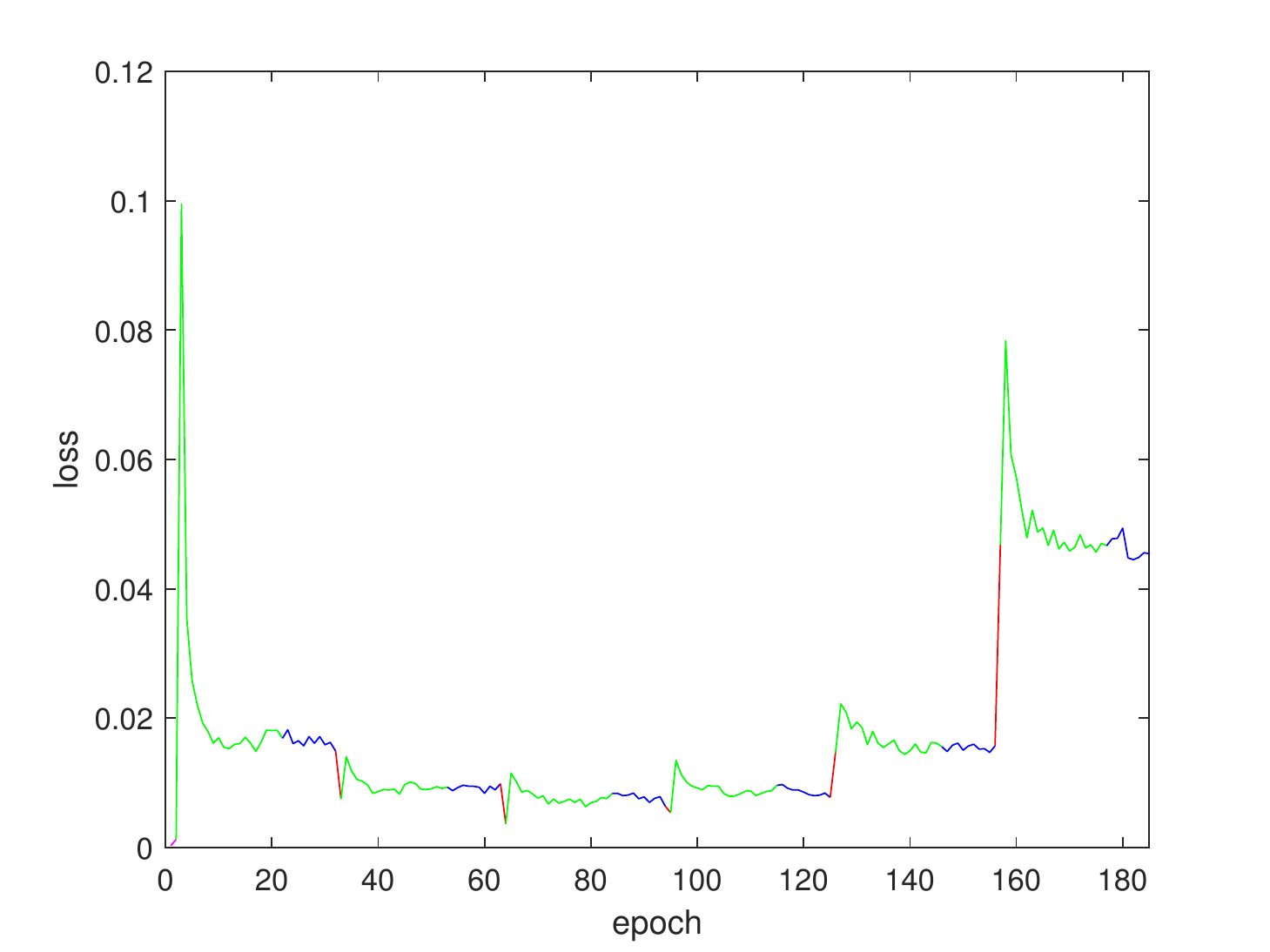}}
        \caption{The train loss of VGG on CIFAR10 trained by \sysname. (a) The train loss of 5 trials. (b) A detailed example train loss. 'Magenta' stands for initialization; 'Green' stands for optimizing $\bm{B}_g$ with speedup; 'Blue' stands for optimizing $\bm{\alpha}_g$; 'Red' stands for pruning in $\bm{\alpha}$ domain. Please see this figure in color.}
        \label{fig:three graphs}
\end{figure}

\subsection{ResNet18/34 on ILSVRC12}
\label{sec:resnet18appendix}
The ImageNet (ILSVRC12) dataset~\cite{ILSVRC15} consists of $1.2$ million high-resolution images for classifying in 1000 object classes. 
The validation set contains 50k images, which are used to report the accuracy level.
We use mini-batch with size of 256. The used ResNet18/34 is from~\cite{bib:CVPR16:He}.
For data preprocessing, we use the setting provided by~\cite{torchIMAGENET}.
We use the ResNet18/34 provided by Pytorch as the baseline of full precision version. 
The network architecture is the same as "resnet18/resnet34" in~\cite{torchResNet}.

The structures of each layer chosen for \sysname are all \textit{pixel-wise} except for the first layer (\textit{kernel-wise}) and the last layer (\textit{subchannel-wise(2)}).
After each pruning, the network is retrained to recover the accuracy degradation with 10 epochs of optimizing $\bm{B}_g$ and 5 epochs of optimizing $\bm{\alpha}_g$. 
The pruning ratio is 15\%, and 5/9 times pruning (Step 1) are executed after initialization in the reported experiments (\tabref{tab:ResNet}). 

For quantizing a large network with an average low bitwidth (\eg $\leq 2.0$), we find that adding our \textit{STE with loss-aware} steps in the end can result an around $1\%\sim2\%$ higher accuracy (see~\tabref{tab:ResNet}) than adding the optimizing steps of \secref{sec:updating}.
Thus, we add another 50 epochs of \textit{STE with loss-aware} in the end for quantizing ResNet18/34. 
The learning rate is $10^{-4}$, and gradually decays with $0.98$ per epoch. 
Here, \textit{STE with loss-aware} is just used in the end to further seeking a higher final accuracy.

We think this is due to the fact that several layers have already been pruned to an extremely low bitwidth ($<1.0$-bit).
With such an extremely low bitwidth, maintained full precision weights help to calibrate some aggressive steps of quantization, which slowly converges to a local optimum with a higher accuracy for a large network.
Recall that maintaining full precision parameters means STE is required to approximate the gradients, since the true-gradients only relate to the quantized parameters used in the forward propagation.
However, for the quantization bitwidth higher than two ($>2.0$-bit), the quantizer can already take smooth steps, and the gradient approximation brought from STE damages the training process inevitably.
Thus in this case, the true-gradient optimizer (our optimization steps in \secref{sec:updating}) can converge to a better local optimum, faster and more stable.

\end{document}